\documentclass{article}

\usepackage{microtype}
\usepackage{graphicx}
\usepackage{subcaption}
\usepackage{booktabs} 
\usepackage{makecell}
\usepackage{hyperref}

\usepackage{placeins}
\usepackage{multirow}

\usepackage{etoc}

\usepackage[preprint]{icml2026}

\usepackage{amsmath}
\usepackage{amssymb}
\usepackage{mathtools}
\usepackage{amsthm}
\usepackage{amsfonts}
\usepackage{tcolorbox}
\usepackage{listings}
\usepackage{hyperref}

\usepackage[capitalize,noabbrev]{cleveref}

\usepackage{tikz}
\usetikzlibrary{arrows.meta,calc}
 
\definecolor{actC}{HTML}{2E7D32}
\definecolor{gradC}{HTML}{C62828}
\definecolor{wC}{HTML}{1565C0}
\definecolor{preC}{HTML}{6A1B9A}
\definecolor{inkC}{HTML}{455A64}
 
\tikzset{
  stage/.style={draw=#1, line width=0.9pt, rounded corners=5pt, fill=#1!7,
                align=center, minimum width=2.9cm, minimum height=1.15cm,
                inner sep=4pt, font=\small},
  cyc/.style={-{Latex[length=2.8mm,width=2.3mm]}, line width=1.3pt, draw=black},
  lab/.style={font=\normalsize, text=black},
}

\theoremstyle{plain}
\newtheorem{theorem}{Theorem}[section]

\theoremstyle{definition}

\theoremstyle{remark}
\newtheorem{remark}[theorem]{Remark}

\makeatletter
\let\old@addtocontentsline\addtocontentsline
\makeatother

\usepackage[textsize=tiny]{todonotes}

\icmltitlerunning{Bug or Feature}

\newcommand{\rev}[1]{#1}

\begin{document}
\twocolumn[
\icmltitle{Bug or Feature$^2$: Weight Drift and Activation Sparsity}

\icmlsetsymbol{equal}{*}

\begin{icmlauthorlist}
\icmlauthor{Egor Shvetsov}{equal}
\icmlauthor{Aleksandr Serkov}{equal}
\icmlauthor{Shokorov Viacheslav}{}
\icmlauthor{Redko Dmitry}{}
\icmlauthor{Vladislav Goloshchapov}{}
\icmlauthor{Evgeny Burnaev}{}
\end{icmlauthorlist}


\icmlcorrespondingauthor{Egor Shvetsov}{geksut@gmail.com}


\vskip 0.3in
]

\printAffiliationsAndNotice{\icmlEqualContribution}

\begin{abstract}
The design of modern neural architectures has converged through
incremental empirical choices, yet the mechanisms governing their
training dynamics remain only partially understood. We identify
and analyze a negative weight drift induced by the interaction
between standard losses and positively biased activation functions.
We prove that under MSE or cross-entropy loss, the gradient with
respect to positive pre-activations is non-negative in expectation
at initialization, driving downstream weights toward negative
values during early training. The drift is intrinsic to
optimization rather than data, and persists across architectures
 and asymmetric activation functions
(ReLU, GELU, SiLU). Coupled with ReLU, weight drift produces
activation sparsity reaching nearly 90\% in GPT-nano. We further analyze whether
this natural sparsity is damaging and find that, on average, models sharply degrade above $\sim$70\%
activation sparsity. We then evaluate sparsity-inducing activation functions, which reveals activation
spikes in intermediate transformer layers.
Code is available at \url{https://github.com/On-Point-RND/BugOrFeature}.
\end{abstract}


\section{Introduction} 
The design of modern  neural architectures resembles an evolutionary process that converges toward stable paradigms. 
Incremental improvements are frequently discovered and adopted in an ad-hoc manner, 
yet we do not fully understand the intrinsic mechanics governing the models we employ. 
In this work, we identify and study a negative drift in weight distributions induced by 
the interaction between standard losses and positively asymmetric activation functions. 
This drift further negatively shifts the mean of intermediate representations. 
Passing these representations through the activation functions squashes them toward zero, 
which in turn reinforces the drift that produced them (Figure~\ref{fig:overview}).

\textbf{Weight drift.} We formally illustrate and empirically verify that, for positively biased
activation functions combined with standard losses,
gradient descent drives weights toward negative values during the early iterations of training.
This shift is largest in the first iterations, when the loss is largest,
and the resulting negative offset persists throughout training.
The effect is intrinsic to the optimization rather than to the data,
since the same drift appears when training on entirely random inputs,
and it holds broadly across architectures.

\begin{figure}[t]
\centering
\begin{tikzpicture}[
  stage/.append style={minimum width=2.3cm, minimum height=1.0cm, inner sep=2.5pt, font=\small},
  lab/.append style={font=\small, inner sep=1.5pt},
  cyc/.append style={line width=1.1pt},
]
\node[stage=actC]  (A) at (0, 1.75)  {\textbf{\color{actC!80!black}Activations}\\[0pt] $\mathrm{ReLU}(p)\ge 0$};
\node[stage=gradC] (B) at (2.55, 0)  {\textbf{\color{gradC!80!black}Gradients}\\[0pt] $\mathbb E\big[\tfrac{\partial\ell}{\partial p}\big]\ge 0$};
\node[stage=wC]    (C) at (0,-1.75)  {\textbf{\color{wC!80!black}Weights}\\[0pt] $\mathbb E[w]<0$};
\node[stage=preC]  (D) at (-2.55, 0) {\textbf{\color{preC!80!black}Pre-activations}\\[0pt] $\mathbb E[p]<0$};

\draw[cyc] (A.east)  to[bend left=30] node[lab, pos=0.0, anchor=south west, align=left, xshift=2pt, yshift=1pt]
      {loss\\[-3pt]{\fontsize{5.6}{6.5}\selectfont(MSE, Cross Entropy)}} (B.north);
\draw[cyc] (B.south) to[bend left=30] node[lab, below right=-2pt] {$-\eta\nabla_w\ell$} (C.east);
\draw[cyc] (C.west)  to[bend left=30] node[lab, below left=-2pt]  {$p = Wh$} (D.south);
\draw[cyc] (D.north) to[bend left=30] node[lab, above left=-2pt]  {ReLU} (A.west);
 
\begin{scope}[x=0.33cm, y=0.85cm, shift={(0.05,-0.62)}]
  \draw[-{Stealth[length=1.5mm]}, black, line width=0.5pt] (-3.4,0) -- (2.7,0) node[right=0pt, font=\small] {$w$};
  \draw[wC!45, line width=0.7pt, dashed, domain=-2.6:2.6, samples=60]
       plot (\x,{exp(-\x*\x/1.2)});
  \fill[wC!22, domain=-3.3:2.0, samples=60]
       (-3.3,0) -- plot (\x,{exp(-(\x+0.8)^2/1.2)}) -- (2.0,0) -- cycle;
  \draw[wC, line width=0.9pt, domain=-3.3:2.0, samples=60]
       plot (\x,{exp(-(\x+0.8)^2/1.2)});
  \draw[-{Latex[length=1.4mm]}, wC!80!black, line width=0.7pt] (0,1.12) -- (-0.8,1.12);
  \draw[black!70, dashed, line width=0.5pt] (0,0) -- (0,1.15);
  \node[font=\small, black, below=-1pt] at (0,0) {$0$};
\end{scope}
\end{tikzpicture}
\caption{\textbf{Negative weight drift is a self-reinforcing loop.}
Under MSE or cross-entropy loss, ReLU activations make the
expected gradient with respect to pre-activations non-negative
(\S\ref{sec:formal_drift}). Gradient descent therefore shifts the weight
distribution toward negative values (inset), which pushes pre-activations below
zero and closes the loop. ReLU maps the shifted pre-activations to zero,
producing activation sparsity
(\S\ref{sec:sparsity_predictor}--\S\ref{sec:activation_functions}).}
\label{fig:overview}
\end{figure}

\textbf{Bug or Feature: emergent sparsity.}
In the absence of centering normalization~\footnote{We discuss modern
architectures that do not use centering in \S\ref{app:normalizations}.}
ReLU maps the negatively shifted pre-activations exactly to zero, inducing hard activation
sparsity of nearly 90\% in GPT-nano. With GELU and SiLU, the same drift pushes
activations toward zero, so low-magnitude outputs dominate intermediate representations.
In both regimes, this raises an immediate question of whether it is a \textbf{Bug or a Feature}.
As a bug, uncontrolled sparsity or magnitude suppression risks degrading model performance
by silencing large parts of the network. As a feature, an effect that emerges without any
explicit regularization could provide a mechanism for computational efficiency or improved
interpretability. The answer depends on whether this suppression, hard or soft, hurts
model performance.

\textbf{When high sparsity is harmful.} To answer this question, we vary the sparsity level
using two approaches, Percentile Centering and Top-K sparsity~\citep{top_k_sae}, and find
that performance stays nearly flat at moderate activation sparsity and collapses abruptly
once sparsity exceeds ${\sim}70\%$. Moreover, we demonstrate that skip connections improve robustness to sparsity.

\textbf{Sparsity-inducing activation functions.}
Having established that weight drift and normalization jointly govern
sparsity in ReLU-based models, we ask whether alternative activation
functions offer a more favorable sparsity--accuracy tradeoff. We
evaluate $\text{ReLU}^2$~\citep{so2022primer}, NoisyReLU~\citep{gulcehre2016noisy},
and SUGARBSiLU~\citep{horuz2025resurrection} as candidates
that may simultaneously enhance sparsity and model performance.
The latter two are specifically designed to overcome the zero-gradient problem of ReLU,
in which neurons with negative pre-activations receive no gradient~\citep{died_relu_problem}.
We find that $\text{ReLU}^2$ achieves lower loss than both ReLU and GELU on
GPT-nano across three seeds, but it is highly sensitive
to the choice of normalization. It works well only with LayerNorm~\citep{ba2016layer}
and RMSNorm~\citep{zhang2019root}, and its performance degrades with BatchNorm or
without normalization.

\textbf{Spike amplification.}
In GPT-nano models, we identify spikes, i.e., large positive activation values, in
early layers, and find that $\text{ReLU}^2$ amplifies them. This raises once more the question
of whether this amplification is a \textbf{Bug or a Feature}$^2$. We find that amplified
spikes may become damaging, whereas clipping the amplified activations mitigates it.

\textbf{An efficiency bonus.}
Finally, the fact that the weight drift happens only during first iterations yields a practical result.
Since the critical dynamics stabilize after only a few iterations, centering statistics, quantile shifts, and Top-K
thresholds can all be computed as running means exclusively over these early steps  enabling savings in compute time without sacrificing effectiveness.

\emph{The paper is organized as follows.}
\S\ref{sec:formal_drift} and \S\ref{sec:weight_drift} formally and empirically
characterize weight drift. \S\ref{sec:sparsity_predictor} analyzes controllable
post-activation sparsity and its relationship to model accuracy.
\S\ref{sec:activation_functions} evaluates sparsity-inducing activation functions, 
while \S\ref{sec:gpt_results} examines activation spikes in GPT-nano models.
\S\ref{sec:efficiency} discusses computational efficiency gains enabled by early
drift stabilization, and \S\ref{sec:related_work} covers related work.
The Appendix covers proofs, implementation details, 
and extended experimental results over different settings and models.

\section{Formal Illustration of Negative Weight
Drift}
\label{sec:formal_drift}
Throughout this section we restrict the formal argument to ReLU~\citep{nair2010relu} and demonstrate results for other activation functions empirically~\footnote{
The formal extension of Theorems~\ref{th:pos_p_gradient} and \ref{th:pos_p_gradient_ce} to smooth activations requires a continuous analogue of the survival-conditioning argument and is beyond the scope of this paper. }.
Consider a multilayer perceptron with randomly initialized, zero-mean
weights, using ReLU activation without mean-centering normalization
layers. We demonstrate that at initialization and during the first training iterations under MSE or cross-entropy loss, the gradient of the
loss with respect to the pre-activations is positive in expectation.
Since gradient descent applies updates in the negative direction of the
gradient ($\mathbf{w} \leftarrow \mathbf{w} - \eta \nabla_{\mathbf{w}}$),
these consistently positive gradients drive downstream weights toward
negative values. This negative weight drift in turn shifts pre-activations
further below zero
reinforcing the effect in a self-amplifying cycle. As training progresses
and gradients diminish, the drift stabilizes.
Our formal analysis applies to the early phase of training, when the properties of the random zero-mean initialization still reasonably hold.

\paragraph{Properties of the Effective Weight Matrix.}
Consider a network with $L$ linear layers interleaved with ReLU activations. For a fixed input $\mathbf{x}$, the activation pattern of each ReLU is fixed, so each activation layer acts as a binary diagonal matrix $\mathbf{D}_l$, where $(\mathbf{D}_l)_{ii} = 1$ if the $i$-th neuron is active and $0$ otherwise. 
Pick any intermediate layer $l$ with pre-activation vector $\mathbf{p}^{(l)}$. All layers after $l$ form the composition:
\begin{equation}
    \mathbf{V}_{\mathrm{eff}}^{(l)}
    = \mathbf{W}_L\,\mathbf{D}_{L-1}\,\mathbf{W}_{L-1}\cdots\mathbf{D}_l\,\mathbf{W}_{l+1},
    \label{eq:v_eff}
\end{equation}
so that the network output can be written as
\begin{equation}
    f(\mathbf{x}) = \mathbf{V}_{\mathrm{eff}}^{(l)}\,\sigma\!\left(\mathbf{p}^{(l)}\right).
    \label{eq:folded}
\end{equation}

\begin{theorem}
\label{th:veff_properties}
Let $\mathbf{V}_{\mathrm{eff}}$ be as in~\eqref{eq:v_eff} with $\mathbf{W}_{L}, \dots, \mathbf{W}_{l+1}$ drawn from a zero-mean i.i.d.\ distribution, and $\sigma = \mathrm{ReLU}$. Denote the rows of $\mathbf{V}_{\mathrm{eff}}$ by $\mathbf{v}_1, \dots, \mathbf{v}_{d_p}$. Then:
\begin{equation}
        \mathbb{E}[\mathbf{v}_{i}] = \mathbf{0},  \forall\, i, \quad
        \mathbb{E}[\langle \mathbf{v}_i, \mathbf{v}_j\rangle] \geq 0, \forall\, i,\, j.
\end{equation}
\end{theorem}
\begin{proof}[Proof of the first statement]
Each row of $\mathbf{V}_{\mathrm{eff}}$ can be written as
$\mathbf{v}_i = \left[\mathbf{D}_{L-1}\,\mathbf{W}_{L-1}\cdots\mathbf{D}_l\,\mathbf{W}_{l+1}\right]^\top [\mathbf{W}_{L}]_{i}$,
where $[\mathbf{W}_{L}]_{i}$ denotes the $i$-th row of $\mathbf{W}_L$. Since $\mathbb{E}[W_{ij}] = 0$ for all entries and the diagonal matrices $\mathbf{D}_k$ are fixed (determined by the input), the expectation factors through the outermost weight matrix, giving $\mathbb{E}[\mathbf{v}_{i}] = \mathbf{0}$.
\end{proof}

\begin{proof}[Proof of the second statement]
For ReLU, each $\mathbf{D}_l$ is a binary diagonal matrix that selects active neurons, so $\mathbf{V}_{\mathrm{eff}}$ is a product of random weight matrices with inactive rows zeroed out. At each ReLU gate $\mathbf{D}_k$, the row survives only if the corresponding pre-activation is non-negative, i.e., its inner product with the layer input is $\geq 0$. The property $\mathbb{E}[\langle \mathbf{v}_i, \mathbf{v}_j \rangle] \geq 0$
thus reduces to a property of random vectors conditioned on ReLU survival: it suffices to show that for any two random vectors $\mathbf{a}$ and $\mathbf{b}$ drawn from a zero-mean i.i.d.\ distribution, conditioned on $\mathbf{a}^\top \mathbf{x} \geq 0$ and $\mathbf{b}^\top \mathbf{x} \geq 0$ for a fixed input $\mathbf{x}$, we have $\mathbb{E}[\mathbf{a}^\top \mathbf{b}] \geq 0$. Intuitively, conditioning on survival forces both vectors to share a positive component along the input direction, inducing a positive correlation. By the symmetry of the distribution we may choose coordinates so that $\mathbf{x} = [1, 0, \dots, 0]$. The conditions $\mathbf{a}^\top \mathbf{x} \geq 0$ and $\mathbf{b}^\top \mathbf{x} \geq 0$ then reduce to $a_0 \geq 0$ and $b_0 \geq 0$, so we can write $\mathbf{a} = [|a_0|, a_1, \dots, a_n]$ and $\mathbf{b} = [|b_0|, b_1, \dots, b_n]$. Then:
\[
    \mathbb{E}[\mathbf{a}^\top \mathbf{b}]
    = \mathbb{E}[|a_0| \cdot |b_0|] + \mathbb{E}\!\left[\sum_{i \geq 1} a_i b_i\right].
\]
The first term is strictly positive since $|a_0|$ and $|b_0|$ are positive random variables. The second term equals zero by the zero-mean i.i.d.\ assumption on the entries. Thus $\mathbb{E}[\mathbf{a}^\top \mathbf{b}] \geq 0$.
\end{proof}

\paragraph{Positive Expected Gradient under MSE and Cross-Entropy.}
\label{subsec:pos_grad}
Here we show that, at initialization, the gradient of the loss with
respect to any positive pre-activation is non-negative in expectation,
both for MSE regression and for softmax cross-entropy classification.
We state the two results in parallel. Their proofs are provided in
Sections~\ref{sec:mse_proof} and~\ref{sec:ce_proof}, respectively.

\begin{theorem}[MSE loss]
\label{th:pos_p_gradient}
Let $f(\mathbf{x}) = \mathbf{V}_{\mathrm{eff}}^{(l)}\,\sigma(\mathbf{p}^{(l)})$
be as in~\eqref{eq:folded}, with $\sigma = \mathrm{ReLU}$ and
$\mathbf{V}_{\mathrm{eff}}^{(l)}$ satisfying
Theorem~\ref{th:veff_properties}. Assume the network is at initialization,
so that $\mathbf{V}_{\mathrm{eff}}^{(l)}$ is independent of
$\mathbf{p}^{(l)}$ and of the target $\mathbf{y}$. Consider the MSE loss
$\ell(f(\mathbf{x}), \mathbf{y}) = \tfrac{1}{2}\|f(\mathbf{x}) - \mathbf{y}\|_2^2$.
Then for any neuron $i$, $\mathbb{E}\!\left[\frac{\partial \ell}{\partial p_i^{(l)}}\right] \geq 0$, with strict inequality whenever $p_i^{(l)} > 0$, where the expectation
is taken over $\mathbf{V}_{\mathrm{eff}}^{(l)}$.
\end{theorem}

\begin{theorem}[Cross-entropy loss]
\label{th:pos_p_gradient_ce}
Let $f(\mathbf{x}) = \mathbf{V}_{\mathrm{eff}}^{(l)}\,\sigma(\mathbf{p}^{(l)})$
be as in~\eqref{eq:folded}, with $\sigma = \mathrm{ReLU}$ and
$\mathbf{V}_{\mathrm{eff}}^{(l)} \in \mathbb{R}^{C \times d_p}$ satisfying
Theorem~\ref{th:veff_properties}, where $C$ denotes the number of output
classes. Assume the network is at initialization, so that
$\mathbf{V}_{\mathrm{eff}}^{(l)}$ is independent of $\mathbf{p}^{(l)}$ and
of the one-hot label $\mathbf{y}$. Consider the cross-entropy loss
$\ell(f(\mathbf{x}), \mathbf{y}) = -\sum_{c=1}^{C} y_c \log s_c$,
where $\mathbf{s} = \mathrm{softmax}(f(\mathbf{x}))$. Then for any
neuron $i$,
 $\mathbb{E}\!\left[\frac{\partial \ell}{\partial p_i^{(l)}}\right] \geq 0$, with strict inequality whenever $p_i^{(l)} > 0$, where the expectation
is taken over $\mathbf{V}_{\mathrm{eff}}^{(l)}$, up to corrections of
order $O(\|\mathbf{f}\|^2)$ from the softmax linearization around
$\mathbf{f} = \mathbf{0}$.
\end{theorem}

\paragraph{Extension to Arbitrary Depth and Locality.}
Theorems~\ref{th:veff_properties} and~\ref{th:pos_p_gradient} hold for any 
intermediate layer $l$, so the positive-gradient property and the 
resulting negative weight drift propagate through the entire network. 
For each $l \in \{1, \dots, L-1\}$, we fold all subsequent layers into 
$\mathbf{V}_{\mathrm{eff}}^{(l)}$ as in~\eqref{eq:v_eff}. By 
Theorem~\ref{th:veff_properties}, $\mathbf{V}_{\mathrm{eff}}^{(l)}$ 
satisfies the zero-mean and non-negative cross-correlation conditions at 
initialization. By Theorem~\ref{th:pos_p_gradient}, the gradient with 
respect to any positive pre-activation $p_i^{(l)}$ at layer $l$ is 
non-negative in expectation over the downstream weights, with strict 
positivity whenever $p_i^{(l)} > 0$. Consequently, weights at every layer 
experience a non-positive expected update, i.e.\ negative drift. This consequently shifts pre-activations at layer $l$ downward, increasing the fraction of neurons falling below zero and reinforcing the cycle. 

Although we present the analysis for a ReLU MLP without normalization
or skip connections, the argument is local, the structural property
established in Theorem~\ref{th:veff_properties} depends only on
the composition of subsequent linear layers and ReLU gates, and applies
to any contiguous stack of such layers within a larger architecture. 
The argument also carries over to architectures with skip 
connections (Appendix~\ref{app:skip_theory}), 
and we discuss its extension beyond initial training steps in Appendix~\ref{app:beyond_init}.

\section{Empirical Results for Negative Weight Drift}
\label{sec:weight_drift}
In the previous section we established that gradient descent drives weights negative in
expectation during early training. We now verify this empirically across
optimizers, learning rates, architectures, and activation functions.

\begin{figure}[ht]
    \centering
    \includegraphics[width=\linewidth]{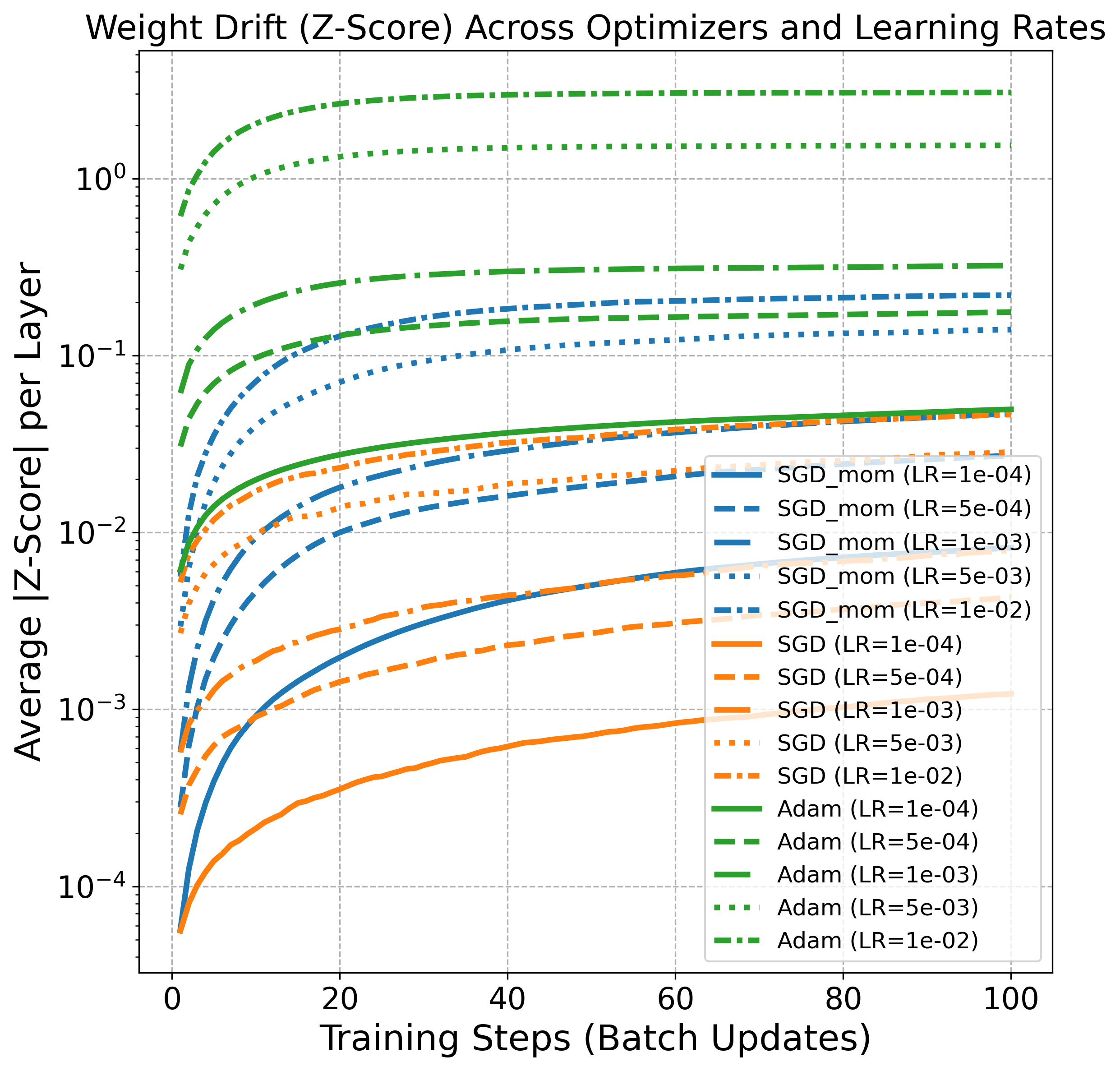}
    \caption{Weight drift measured as average absolute Z-score per layer over the
    first 100 training steps for an MLP trained on CIFAR-10. Momentum accelerates
    initial weight changes and leads to rapid convergence toward asymptotic drift levels,
    while plain SGD exhibits slower progression. Results are in log scale.}
    \label{fig:z_score_optimizers}
\end{figure}

\paragraph{Weight drift depends on optimizer and learning rate.}
First, we  analyze how quickly model weights evolve during the first
batch updates and what weight drift depends on.
Let $w_0 \in \mathbb{R}$ denote a scalar weight at initialization and $w_t$ its value
after $t$ gradient steps. We measure relative drift per layer using the \emph{Z-score}=$\mathbb{E}\!\left[|w_t - w_0| / \mathrm{std}(w_0)\right]$, where the expectation is taken
over all weights in a layer and the absolute value captures drift in both directions
symmetrically. Results in Figure~\ref{fig:z_score_optimizers} demonstrate that: 
\textbf{(1)}~momentum substantially accelerates drift, \textbf{(2)}~within each
optimizer, higher learning rates produce faster and larger drift and \textbf{(3)}~momentum-based
optimizers exhibit a rapid initial surge that then plateaus, while plain SGD progresses more
slowly and near-linearly.  We hypothesize that momentum amplifies positive gradient
bias at early accumulation steps.  Additional results on training dynamics are presented 
in  \S\ref{app:additional_drift}.

\begin{figure*}[!ht]
    \centering
    \begin{subfigure}{0.95\linewidth}
        \centering
        \includegraphics[width=\linewidth]{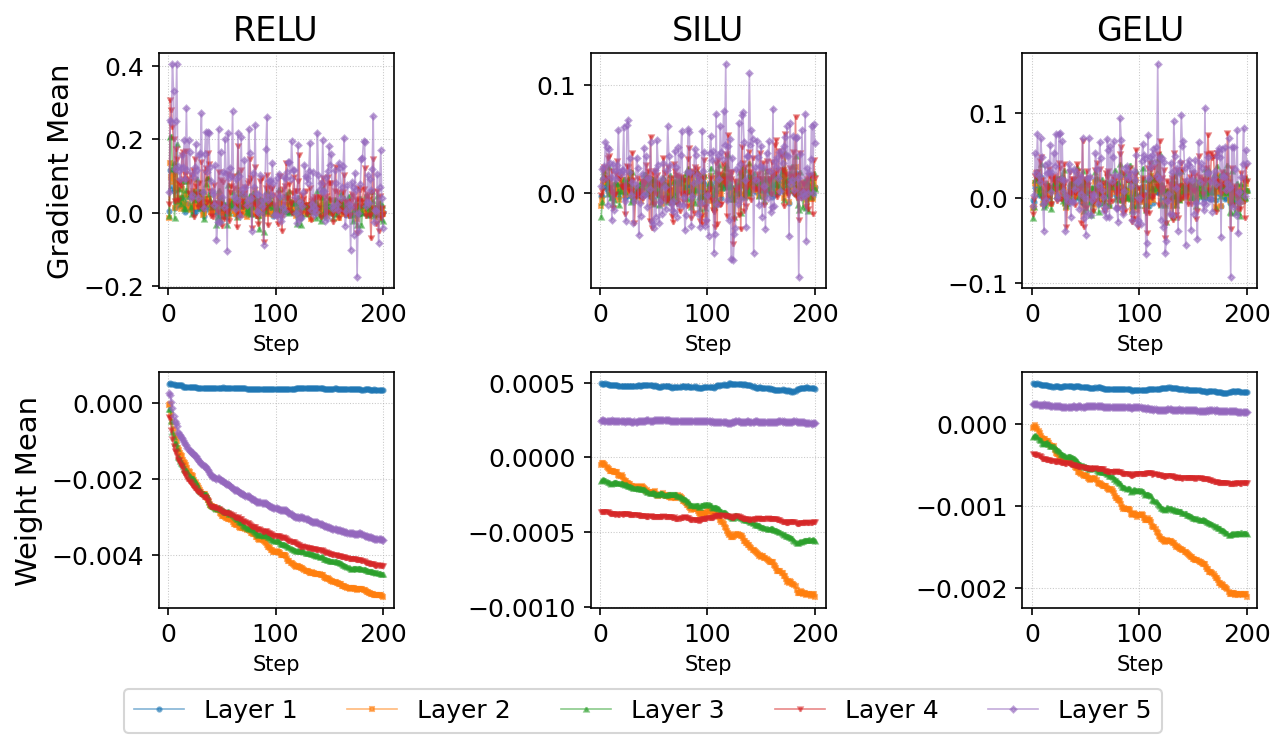}
        \caption{Random inputs, MSE loss, SGD.}
        \label{fig:mlp_weight_drift}
    \end{subfigure}
    \vspace{0.5em}

    \begin{subfigure}{0.95\linewidth}
        \centering
        \includegraphics[width=\linewidth]{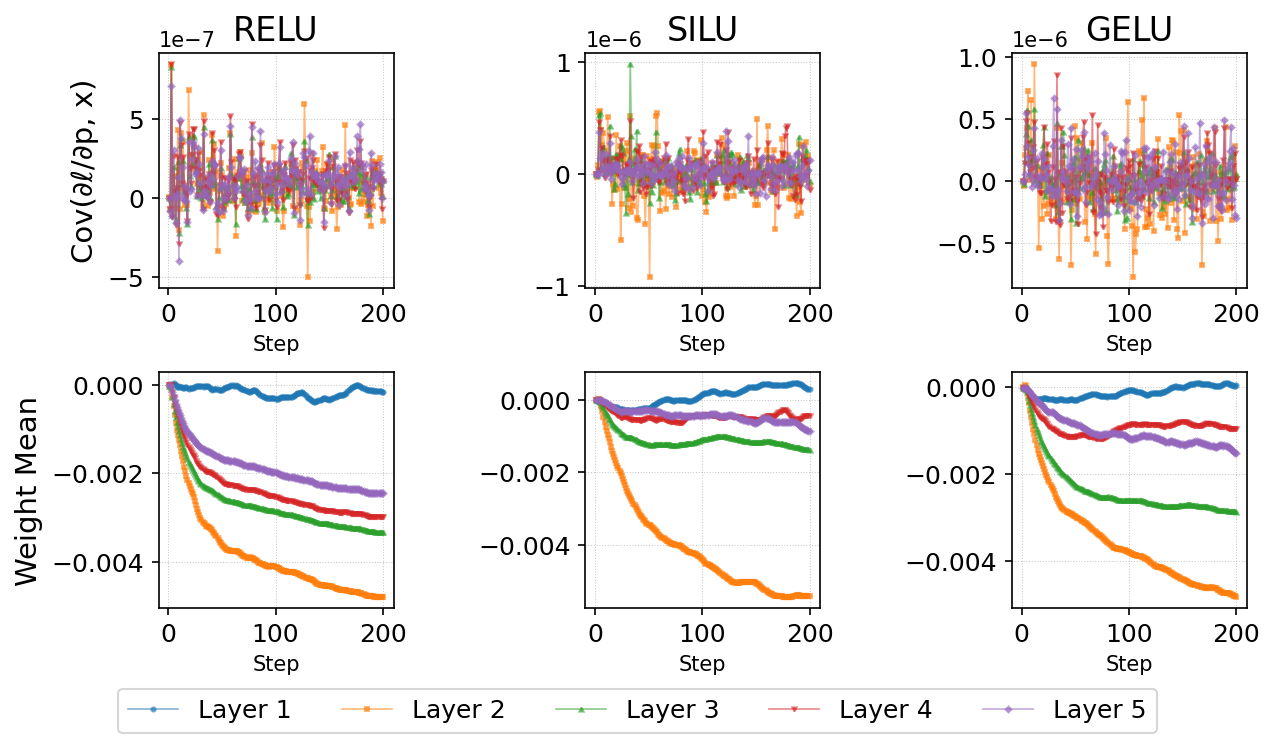}
        \caption{CIFAR-10, cross-entropy loss, Adam.}
        \label{fig:mlp_ce_cov}
    \end{subfigure}
    \caption{\textbf{Negative weight drift in a five-layer MLP without normalization.}
    \textbf{(a)}~Training on random $\{X,Y\}$ pairs sampled from $\mathcal{N}(0,1)$, 
    the top row shows the mean gradient and the bottom row the mean weight per layer.
    \textbf{(b)}~Training on CIFAR-10, the top row shows the covariance term
    $\mathrm{Cov}_{\mathbf{x}}(\partial\ell/\partial p_i, x)$, which is three to four orders of
    magnitude smaller than the weight mean (bottom row).
    Trajectories are averaged across 5 runs.
    Weights drift negative for all three activation functions.
   Weights of Layer~1, whose input has not passed through an activation, do not drift, as predicted. 
   Further details are given in \S\ref{sec:mlp_details}.}
    \label{fig:mlp_combined}
\end{figure*}
\begin{figure*}[h]
    \centering
    \includegraphics[width=0.9\linewidth]{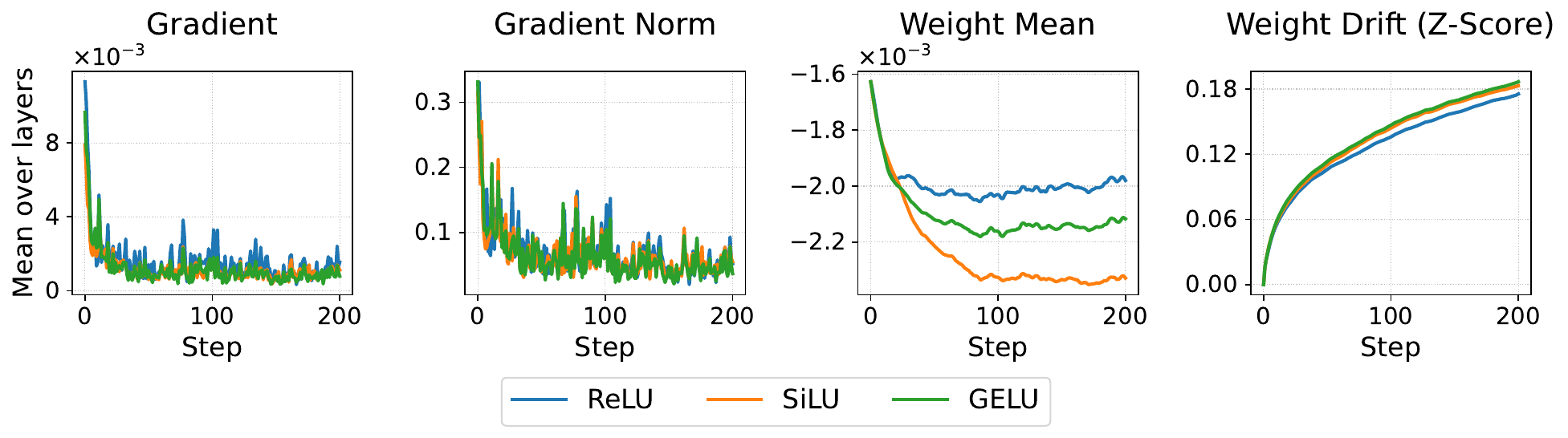}
    \caption{
\textbf{Weight drift in MP-SENet.}
Training dynamics under GELU, ReLU, and SiLU with the AdamW optimizer
($\text{lr} = 5 \times 10^{-4}$), averaged across all model layers. Drift patterns are
qualitatively consistent with those observed in the MLP and ResNet settings. 
    }
    \label{fig:senet}
\end{figure*}

\paragraph{Drift is intrinsic to optimization, not data.}
To demonstrate that negative weight drift is an intrinsic property
of the optimization process, we train the same MLP on random $\{X, Y\}$ pairs sampled from
$\mathcal{N}(0, 1)$. As shown in Figure~\ref{fig:mlp_weight_drift}, negative drift
arises across all activation functions even on entirely random data.
For ReLU, whose positive activation bias is strictly enforced at every layer, the drift
exhibits a clear depth ordering, where earlier layers accumulate more negative weights.
For SiLU\rev{~\citep{elfwing2018silu}} and GELU\rev{~\citep{hendrycks2016gaussian}}, whose
outputs are only positively \emph{biased} rather than strictly non-negative, the drift is
weaker and the last layer barely drifts, though the overall drift remains. 
First layers do not experience drift since their inputs have not yet passed through an activation function.

\paragraph{Negative weight drift across architectures and activation functions.}
Theorems~\ref{th:pos_p_gradient} and~\ref{th:pos_p_gradient_ce} predict
a positive expected gradient at initialization, and we find that this prediction holds broadly
in practice. Across four architectures, MLP, MaxViT-Tiny~\citep{tu2022maxvit},
MP-SENet~\citep{lu2023mp} shown in Figure~\ref{fig:senet},
and ResNet-18~\citep{he2016resnet} shown in Appendix Figure~\ref{fig:resnet_drift}, the
positive-gradient property is observed. For the MLP and ResNet-18, where we measure it, the covariance term
$\mathrm{Cov}_{\mathbf{x}}(\partial\ell/\partial p_i, x)$ remains orders of magnitude smaller
than the weight mean, validating the assumptions made in Appendix~\ref{sec:mse_proof}.
The results are consistent across the activation functions GELU, ReLU, and SiLU, and also hold for models with skip connections.
Adding a skip connection to the MLP leaves the drift trajectory qualitatively
unchanged under all four tested activation functions (Appendix Figure~\ref{fig:mlp_training_dynamics}).
The residual architectures ResNet-18 and MaxViT also exhibit the predicted positive gradient bias.

Table~\ref{tab:activations} reports accuracy and the fraction of negative pre-activation values for MLP, ResNet, ViT, and GPT across six activation functions. As a direct consequence of negative weights, the fraction of negative pre-activations is substantial in nearly every configuration typically between 60\% and 80\% confirming that weight drift is present throughout. The one exception is ResNet with batch normalization, where mean-centering directly disrupts the drift.


\section{Post-activation Sparsity and Performance}
\label{sec:sparsity_predictor}

\paragraph{Controllable Sparsity.} Weight drift naturally induces hard activation sparsity in ReLU-based models,
 while for smooth activations such as GELU and SiLU, it pushes pre-activations into near-zero regions, 
 yielding predominantly low-magnitude outputs. Since this behavior emerges directly from the optimization 
 dynamics, we want to evaluate whether the resulting sparsity impairs model performance 
 or is instead beneficial. To answer this, 
 we control post-activation sparsity levels across architectures and investigate whether explicitly enforcing higher or lower sparsity improves downstream performance. We further examine whether an optimal sparsity regime exists that outperforms the baseline naturally induced by weight drift.

\paragraph{Sparsification mechanisms.} We consider two  
methods to control post-activation sparsity, one commonly used in 
the literature and one we propose specifically to evaluate 
ReLU-induced sparsity at controlled levels. \emph{Top-K} activation sparsity~\citep{top_k_sae} retains the  $k\%$ largest activations and hard-zeros the rest. We pair Top-K 
with GELU rather than ReLU, since ReLU already induces sparsity 
via weight drift, making it impossible to reliably control the 
sparsity lower bound. To evaluate controlled sparsity in ReLU-based models directly, 
we propose \textbf{Percentile Centering} (PC), which integrates 
into existing normalization layers with minimal architectural 
changes. Rather than shifting activations by the mean as in 
standard BatchNorm (BN) or LayerNorm, PC shifts by a target percentile 
$q$, $\hat{x} = \frac{x - Q_x(q)}{\sqrt{\sigma_x^2 + \epsilon}}$,
where $Q_x(q)$ denotes the $q$-th percentile of the pre-activation 
distribution and $\sigma_x^2$ its variance. When followed by ReLU, 
this causes a $q\%$ fraction of activations to fall below zero, 
directly controlling post-activation sparsity. This is particularly 
convenient for architectures like ResNet, where BN is 
placed immediately before ReLU. In our implementation, PC maintains a running 
mean of the percentile estimate analogous to the running statistics 
in BN.

\paragraph{Experimental setup.} We evaluate both mechanisms across 
four architectures: MLP, ResNet-18, MaxViT-Tiny, and GPT-nano~\citep{karpathy2022nanogpt}. For 
ResNet-18, we additionally distinguish between \textit{per-activation} 
sparsity, which zeros individual activation values, and 
\textit{per-channel} (structured) sparsity, which zeros entire 
feature map channels with the latter being more relevant for 
practical applications. In total, we obtain 
$N = 79$ (model, sparsity) pairs across architectures, activation 
functions, and sparsification mechanisms. \emph{Sparsity is measured exclusively at the output of activation functions. Consequently, we do not account for intermediate model components that lack activation functions, for example, attention layers in transformers.}
Technical details are outlined in \S\ref{app:technical_details}.

\subsection{Results for controlled post-activation sparsity experiments}
While raw numerical results are 
presented in Tables~\ref{tab:pruning_results_resized} 
and~\ref{tab:pruning_results_extended} to enable comparison across architectures, we normalized performance metrics by the maximum value observed for each architecture--modification pair. Since we report accuracy for all models except GPT-nano (where we use loss), we inverted the loss values for GPT-nano so that higher values consistently indicate better performance.  Visual inspection of the scaled metrics in Figure~\ref{fig:sparsity_predictor} and
 Appendix Figure~\ref{fig:sparsity_per_arch} (each architecture separately), reveals that performance remains largely stable across moderate sparsity levels and degrades sharply only beyond a high threshold. We further fit a three-parameter power law via nonlinear least squares: $\hat{a}(s) = A - B \cdot s^{N}$,
   where $s \in [0,1]$ denotes sparsity, and $A$, $B$, $N$ are free parameters. 
   The fitted coefficient $A = 0.978$ corresponds to the predicted performance 
   at zero sparsity, confirming near-complete retention of accuracy when no activations are zeroed. $B = 0.635$ 
   represents the maximum potential drop.  The exponent $N = 16.72$ sets the sharpness of the transition,
     $s^N$ remains 
   negligible for $s \lesssim 0.7$ and grows rapidly thereafter so accuracy degrades. Together, these values quantify
    the qualitative ``cliff'' visible in Figure~\ref{fig:sparsity_predictor}, with performance preserved across 
    a wide plateau before abruptly degrading at high sparsity.

The position of this cliff is primarily determined by model architecture,
 with skip connections improving robustness. For example, a plain MLP 
 drops to chance accuracy ($10\%$) at $95\%$ sparsity, while the MLP with a skip connection still reaches $27.1\%$. Transformers, especially GPT-nano, exhibit extreme resilience, with 
    loss remaining nearly flat up to $s \approx 0.91$. This robustness is partially attributed to skip
     connections, however, further investigation is required. In ResNet-18, channel-wise
      sparsity degrades more monotonically than unstructured Top-K sparsity, meaning that unstructured sparsity 
      causes less model degradation.
\begin{tcolorbox}[
    colback=white,
    colframe=black,
    boxrule=0.5pt,
    before skip=6pt,
    after skip=6pt
]
Empirical results reveal an \textbf{accuracy cliff} with a small preceding monotonic decline. Performance remains nearly flat for moderate sparsity ($s \lesssim 0.7$) before collapsing abruptly once a critical threshold is reached. MLP and ResNet-18 with structured sparsity exhibit more monotonic preceding degradation.
\end{tcolorbox}
\begin{figure}[!ht]
    \centering
    \includegraphics[width=\linewidth]{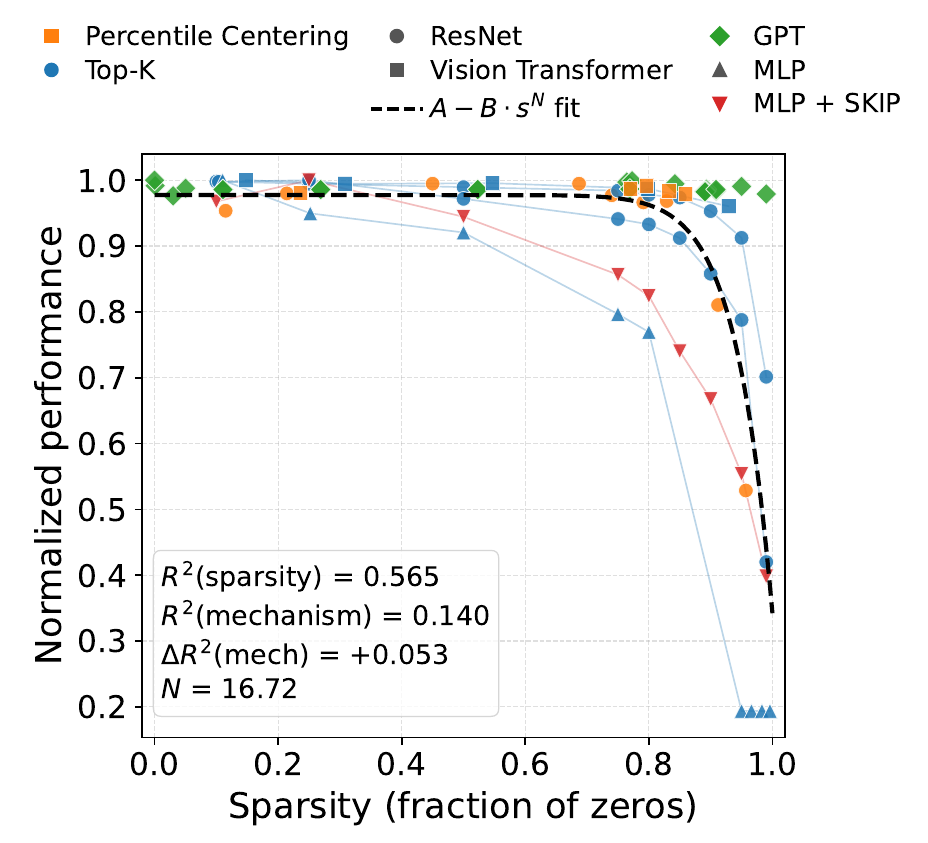}
    \caption{Scaled performance versus post-activation sparsity. The dashed curve denotes 
    the fitted power-law decay model, 
    with sparsity level as the dominant predictor 
    of accuracy $R^2 = 0.565$, while the choice of mechanism (TOP-K vs.\ PC) contributes only a marginal incremental gain $\Delta R^2 = +0.053$.}
    \label{fig:sparsity_predictor}
\end{figure}


Finally, we observe that the specific \emph{mechanism} of sparsification  
Top-K vs. Percentile Centering is of secondary importance,
 adding a mechanism indicator to the power law model marginally improves $R^2$ from $0.565$ to $0.618$, 
 the \emph{amount} of sparsity remains the dominant determinant of predictive accuracy.

\begin{table*}
\centering
\caption{Accuracy (or \rev{test} loss for GPT) and fraction of
negative pre-activations across architectures and activation
functions. For MaxViT-Tiny  we replace only default GELU activations with
the specified function. Best per-row sparsity and performance are
in \textbf{bold}. Details on  
measurement are provided in 
\S~\ref{app:statistics}. \textbf{Format: Accuracy / Negative values.}} 
\resizebox{1\textwidth}{!}{%
    \begin{tabular}{lcccccc}
    \toprule
    \textbf{Model}
        & \textbf{ReLU}
        & \textbf{GELU}
        & \textbf{SiLU}
        & \textbf{NoisyReLU}
        & \textbf{SUGARBSiLU}
        & $\textbf{ReLU}^2$ \\
    \midrule
    MLP
        & 46.6\% / 70.8\%
        & \textbf{51.1}\% / 63.6\%
        & 49.8\% / 61.0\%
        & 49.1\% / 70.4\%
        & 40.8\% / \textbf{72.5}\%
        & 10.0\% / \phantom{0}8.0\% \\
    MLP + RMSNorm
        & 44.7\% / 68.4\%
        & 48.4\% / 64.5\%
        & \textbf{49.4}\% / 62.3\%
        & 48.4\% / 66.8\%
        & 48.6\% / \textbf{70.8}\%
        & 48.0\% / 65.3\% \\
    MLP + LayerNorm
        & 40.9\% / \textbf{68.6}\%
        & 42.3\% / 67.2\%
        & 43.9\% / 65.4\%
        & 41.8\% / 66.4\%
        & 41.0\% / 58.5\%
        & \textbf{45.3}\% / 62.5\% \\
    ResNet (BN)
        & 93.8\% / 43.6\%
        & \textbf{94.1}\% / 43.0\%
        & 93.3\% / 43.6\%
        & 93.5\% / 43.1\%
        & 92.1\% / \textbf{55.5}\%
        & 10.0\% / \phantom{0}0.0\% \\
    MaxViT-Tiny 
        & 69.6\% / 79.8\%
        & \textbf{70.3}\% / 70.5\%
        & 69.4\% / 70.2\%
        & 68.5\% / 79.7\%
        & 51.5\% / \textbf{82.9}\%
        & 62.4\% / 77.1\% \\
    GPT (loss$\downarrow$)
        & 3.290 / 89.0\%
        & 3.260 / \textbf{89.3}\%
        & \multicolumn{1}{c}{--}
        & 3.287 / 89.2\%
        & 3.291 / 89.1\%
        & \textbf{3.251} / 84.2\% \\
    \midrule
    \textit{Average:} 
        & 59.1\% / 66.2\%
        & \textbf{61.2}\% / 61.8\%
        & \textbf{61.2}\% / 60.5\%
        & 60.3\% / 65.2\%
        & 54.8\% / \textbf{68.0}\%
        & 51.9\% / 68.3\% \\
    \bottomrule
    \end{tabular}%
}
\label{tab:activations}
\end{table*}

\begin{table*}[ht]
    \centering
    \caption{Loss and per-layer down-projection input range $r = \max - \min$ at different
GPT-nano layers. 
    For ReLU$^2$, the seed with the largest spikes is shown. All seeds are reported in Appendix Table~\ref{tab:relu2_seeds}.}
    \label{tab:activations_range}
    \begin{tabular}{lcccccccc}
    \toprule
        & \textbf{ReLU} & \textbf{GELU} & \textbf{NReLU} & \textbf{SUGARBSiLU} & $\textbf{ReLU}^2$ & $\textbf{GELU}^2$ & $\textbf{ReLU}^2_{\text{clip15}}$ & $\textbf{ReLU}^2_{\text{clip50}}$ \\
    \midrule
    \rev{Test} loss $\downarrow$: & 3.290 & 3.260 & 3.287 & 3.291 & 3.251 & \textbf{3.233} & 3.242 & 3.236 \\
    \midrule
    \quad Layer 1 & 16.2 & 12.4 & 23.3 & 23.7 & \textbf{108.5} & 26.3 & 15.0 & 50.0 \\
    \quad Layer 2 & 42.2 & 42.7 & 75.3 & 58.7 & \textbf{1055.0} & 59.8 & 15.0 & 50.0 \\
    \quad Layer 3 & 67.5 & 28.5 & 24.2 & 69.0 & \textbf{260.9} & 156.4 & 15.0 & 50.0 \\
    \quad Layer 4 & 17.8 & 10.2 & 15.3 & 13.7 & 49.8 & 124.9 & 15.0 & 50.0 \\
    \quad Layer 5 & 13.1 & 16.0 & 20.3 & 26.5 & 65.3 & 31.0 & 15.0 & 28.1 \\
    \bottomrule
    \end{tabular}

\end{table*}

\section{Sparsity and Activation Functions}
\label{sec:activation_functions}
In \S\ref{sec:weight_drift} we established that weight drift
naturally induces sparsity in ReLU-based models. \emph{We now ask whether
alternative activation functions can achieve a more favorable
sparsity--accuracy tradeoff, either by inducing sparsity through
different mechanisms or by recovering ReLU-style sparsity
post-training.} We evaluate six activation functions, including GELU and ReLU
baselines across four architectures (MLP, ResNet-18, ViT,
GPT-nano).



\paragraph{Candidate activation functions.}
\textit{(1) NoisyReLU}~\citep{gulcehre2016noisy} injects input-dependent
noise into negative pre-activations during training, maintaining
gradient flow through otherwise dead
neurons~\citep{died_relu_problem} while reverting to standard ReLU
at inference. \textit{(2) SUGARBSiLU}~\citep{horuz2025resurrection}
applies ReLU in the forward pass but substitutes a smooth surrogate
gradient (B-SiLU) in the backward pass, ensuring nonzero gradient
signal for negative pre-activations. \textit{(3) ReLU}$^2$, originally
discovered through neural architecture search~\citep{so2022primer},
has since been shown to achieve a strong sparsity--accuracy
tradeoff in LLMs~\citep{zhang2024relu2}. Implementation details for
all three are provided in \S\ref{app:activation_implementations}. 

\textbf{ReLUfication}~\citep{mirzadeh2023relu, author2024calibration}
is not an activation function, but it reaches the same goal, 
so we also consider this procedure.  Rather than training with a
sparsity-inducing activation function from scratch, models trained with a
smooth activation (e.g., GELU) are converted to ReLU variants via
brief fine-tuning. The motivation is that smooth activations avoid
the dying-neuron problem during training, while post-hoc conversion
recovers inference-time sparsity.

\subsection{Results for sparsity and activation functions}
Results are presented in Table~\ref{tab:activations}. 
GELU is the strongest general-purpose baseline.
It achieves the highest accuracy on four of five
classification settings and the second-best loss on GPT-nano,
however, it does not produce sparsity. ReLU$^2$ is normalization-sensitive but excels on GPT.
It collapses entirely on MLP without
normalization and ResNet with BN, but recovers when used with
RMSNorm  or LayerNorm. 
SUGARBSiLU produces the highest sparsity at the cost of stability.
 We attribute this to persistently noisy
gradient trajectories throughout training (Figure~\ref{fig:mlp_training_dynamics} in 
Appendix), suggesting the surrogate
gradient introduces optimization instability. NoisyReLU matches ReLU's sparsity with comparable accuracy.
NoisyReLU achieves average sparsity  close to ReLU's
and competitive accuracy. The
margins are small but consistent, suggesting NoisyReLU is a
viable drop-in replacement when gradient flow through dead
neurons is desired.  

\begin{table} [h]
\centering
\caption{ReLUfication results. Models trained with GELU are
converted to ReLU variants via one epoch of fine-tuning. Format:
\textbf{Accuracy} / \textbf{Negative pre-activations} /
\textbf{Post-activation sparsity}.}
\resizebox{0.43\textwidth}{!}{%
    \begin{tabular}{lll}
    \toprule
    \textbf{Model} & \textbf{Activation} & \textbf{Acc. / Neg.\ Val. / Spar.} \\
    \midrule
    \multirow{2}{*}{MLP}
        & GELU (baseline)   & 51.49\% / 0.632 / 0.026 \\
        & ReLU (1 epoch FT) & 51.55\% / 0.704 / 0.704 \\
    \midrule
    \multirow{3}{*}{ResNet}
        & GELU (baseline)        & 94.03\% / 0.583 / 0.080 \\
        & ReLU (1 epoch FT)      & 93.05\% / 0.577 / 0.578 \\
        & NoisyReLU (1 epoch FT) & 92.31\% / 0.553 / 0.555 \\
    \midrule
    \multirow{2}{*}{MaxViT-Tiny}
        & GELU (baseline)   & 70.30\% / 0.705 / 0.015 \\
        & ReLU (1 epoch FT) & 70.15\% / 0.740 / 0.740 \\
    \bottomrule
    \end{tabular}%
}
\label{tab:relufication}
\end{table}

As shown in Table~\ref{tab:relufication}, ReLUfication yields 55--74\% post-activation sparsity, while accuracy drops by
less than 2 percentage points. The fraction of negative pre-activations changes
little after conversion, and the resulting sparsity closely matches it.
ReLUfication therefore barely changes the weights, weight drift has already
pushed activations below zero such that GELU outputs
values close to zero, and ReLU simply sets these low-magnitude activations to zero.
\begin{tcolorbox}[
    colback=white,
    colframe=black,
    boxrule=0.5pt,
    before skip=3pt,
    after skip=3pt
]
\emph{ReLUfication is the practical recommendation for the sparsity--accuracy tradeoff.}
Rather than training from scratch with ReLU or other sparsity-inducing activations,
light fine-tuning from a GELU baseline creates high post-activation
sparsity with an accuracy loss of less than 2 percentage points, and none for some models,
sidestepping the optimization difficulties of training with ReLU.
\end{tcolorbox}

\section{Spike Amplification with Squared Activation Functions}
\label{sec:gpt_results}

Our layer-wise analysis reveals the presence of spikes or large positive activation
values in early GPT-nano layers regardless of the choice of activation function.
We visualize these spikes in Figure~\ref{fig:spikes_and_activation} (right) and quantify
them in Table~\ref{tab:activations_range}. We measure spikes as  input range $r = \max - \min$ at different
layers. Since the down-projection input follows the activation function, its
minimum is close to zero and $r$ is essentially the maximum activation for that layer. 
While prior work has shown that training
GPT-like models with ReLU$^2$ can improve performance~\citep{so2022primer}, we show that
it also amplifies spikes. A possible cause of this improvement is that squaring amplifies the
large activations of a few important tokens, which can make attention sharper. However, attention
that collapses almost entirely onto these tokens may degrade performance.
We further show that clipping ReLU$^2$ activations removes the spikes without degrading performance.

\begin{figure*}[ht!]
    \centering
    \includegraphics[width=0.6\linewidth]{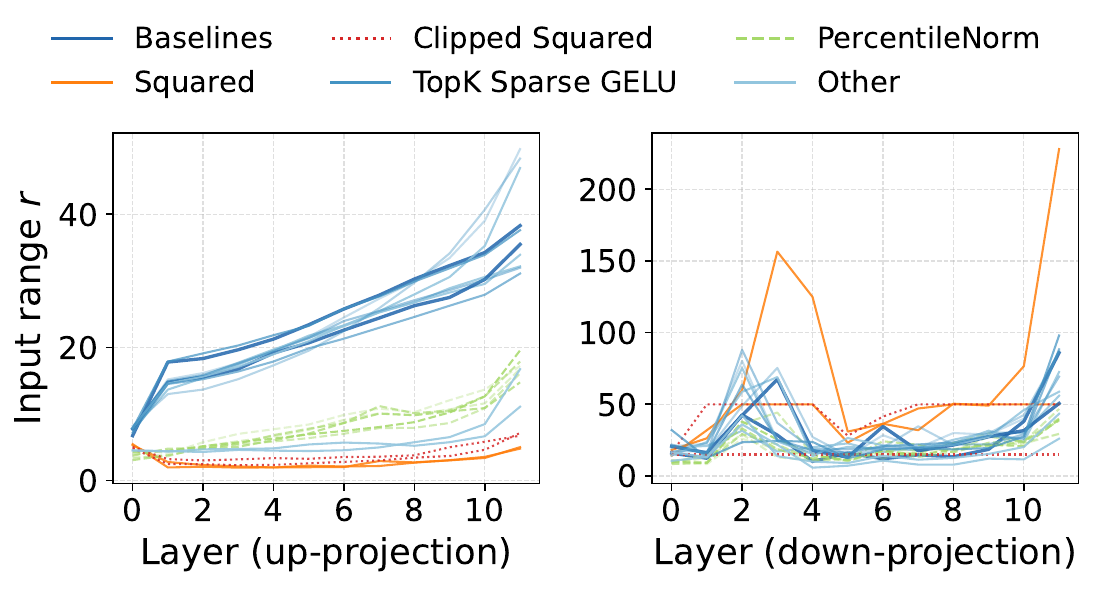}
    \caption{Input ranges into up- and down-projections across layer indices, aggregated over 22 runs for different activation functions, normalization and sparsification strategies.}
    \label{fig:spikes_and_activation}
\end{figure*}

\paragraph{Origin of spikes:}

In our architecture, the MLP block takes the following form:

\noindent\begin{minipage}{\linewidth}
\begin{verbatim}
def forward(self, x):
    x = self.up_projection(x)
    x = self.activation(x)
    x = self.down_projection(x)
    return x
\end{verbatim}
\end{minipage}
\smallskip

\noindent 
As shown in Figure~\ref{fig:spikes_and_activation}, there are no spikes in the inputs to
the \textbf{up-projection} (left), they emerge in the values entering the \textbf{down-projection} (right), precisely
after the activation is applied. Since there is no gating interaction in our MLP block,
the spike must originate in the up-projection itself, where a small subset of neurons
produces anomalously large pre-activation values, which are further amplified
by the nonlinearity. We perform an extended statistical analysis of weights and
activations over all 23 runs in \S\ref{app:gpt_extended}, which further supports this 
conclusion, and discuss possible mechanisms behind the spikes. Profiling six pretrained models shows that distinct early-layer spikes appear only in 
language models, but not in vision models, even with similar architectures (Appendix~\ref{app:pretrained_spikes}).

Figure~\ref{fig:spikes_and_activation} reveals that spike ranges shrink when
normalization with centering or quantile shift is applied. The same pattern holds
for input ranges to the up-projection, which are also smaller under squared activations
and under models that use centering or quantile shift. This points to an intricate
interaction between normalization layers and activation functions, on top of the
weight-drift mechanism discussed in \S\ref{sec:weight_drift}. 

\paragraph{Squared activations amplify spikes.} As shown in Table~\rev{\ref{tab:activations_range}}, baseline
activations (ReLU, GELU) exhibit moderate spikes at Layer~2. ReLU$^2$
amplifies this spike by a factor of 4 to 25 across three seeds, reaching values above
1000 in the run with the largest spike. Exact values for all seeds are given in Appendix
Table~\ref{tab:relu2_seeds}. 

Amplification can be avoided with simple clipping, 
so we test whether it would degrade performance or can be used safely.
We clip ReLU$^2$ activations at thresholds of 15 and 50. We find that ReLU$^2_{\text{clip50}}$ removes the spikes without degrading performance, and it reaches
the same loss in all three seeds, whereas the loss of unclipped ReLU$^2$ varies from seed to seed
(Appendix Table~\ref{tab:relu2_seeds}). 
ReLU$^2_{\text{clip15}}$, however, degrades performance relative to ReLU$^2_{\text{clip50}}$,
suggesting that aggressive clipping suppresses informative large activations.



\section{From Weight Drift to Computational Efficiency}
\label{sec:efficiency}
Dynamic computation of percentiles and centering statistics at every forward pass
introduces non-negligible overhead. Quantile algorithms are poorly parallelizable on
GPUs~\citep{cederman2010gpu, satish2009designing}, and even standard LayerNorm is a
measurable bottleneck~\citep{kanavalau2026gated, ryabykin2026damn}.
 Fortunately, the dynamic computation is not strictly
necessary throughout training. In \S\ref{sec:weight_drift} we demonstrated that weights
stabilize after the initial iterations, and we exploit this property by replacing the
dynamic statistics with fixed values once the network has settled.

Concretely, during a warm-up phase we track centering and percentile statistics with an
EMA scheme analogous to the running statistics in BatchNorm~\citep{ioffe2015batch}:
$\bar{v}_i = \gamma \bar{v}_{i-1} + (1-\gamma) v_i$ with $\gamma = 0.9999$. After
$T_{\mathrm{warm}}$ steps we apply \emph{Accumulation Stop} (AS), freezing statistics
and eliminating dynamic computation. We evaluate this approach on DiT-S/2
and MaxViT, both of which operate on fixed-size inputs. We note that this approach
is less directly applicable to autoregressive models such as GPT, where activation
statistics vary with sequence position and context length.
The baseline uses PyTorch's optimized LayerNorm and BatchNorm kernels, while our AS
implementation uses plain PyTorch without custom CUDA kernels, so all reported throughput gains
are conservative lower bounds. Hardware details are given in Appendix~\ref{sec:common_setup}.

\textbf{Warm-up and measurement protocol.}
The warm-up phase runs for $T_{\mathrm{warm}}=50{,}000$ steps 
for DiT-S/2. For MaxViT-Tiny, warm-up and throughput measurement 
follow the four-phase protocol described in 
\S\ref{sec:maxvit_details}. For DiT-S/2, throughput 
is measured as batches per second averaged over 1{,}000 steps 
after warm-up, with GPU synchronization performed before 
measurement.

\label{app:efficiency_extended}

\textbf{Throughput results.}
During warm-up, computing percentile statistics with our naive PyTorch implementation
reduces throughput by up to ${\sim}28\%$. Once statistics are frozen, throughput returns to
the baseline level on DiT-S/2 ($1.00$--$1.03\times$) and exceeds it on MaxViT-T
($1.09$--$1.34\times$), likely because frozen statistics skip the per-batch reduction that
even the optimized BatchNorm baseline performs (Figure~\ref{fig:accumulation_stop}, left,
Appendix Table~\ref{tab:throughput}).


\paragraph{Quality preservation.}
Figure~\ref{fig:accumulation_stop} (middle, right) and Appendix Table~\ref{tab:quality} show that model quality stays close to the baseline after freezing
statistics. For DiT-S/2 we report FID and IS computed on 50{,}000 generated samples
using the EMA model with Classifier-Free Guidance ($s = 1.5$) and 250 sampling steps.
Interestingly, we observe an improvement when DiT is trained with Percentile
Shift instead of standard LayerNorm, with FID dropping from 49.40 to 48.21. Given
the 50\% shift, the only difference between Percentile LN and standard LN is that the
former uses quantile-based shifting while the latter uses mean centering. One detail
worth noting is that quantile computation in PyTorch propagates gradients only through
a single quantile value. While we do not yet have a full explanation for this result,
the direction warrants its own investigation. For MaxViT-Tiny we report top-1 accuracy
on the ImageNet-1K validation set, all metrics slightly degrade within 2\% relative change of
the corresponding baseline from 70.30\% to 69.75\%. 

\begin{figure*}[t]
    \centering
    \includegraphics[width=0.9\textwidth]{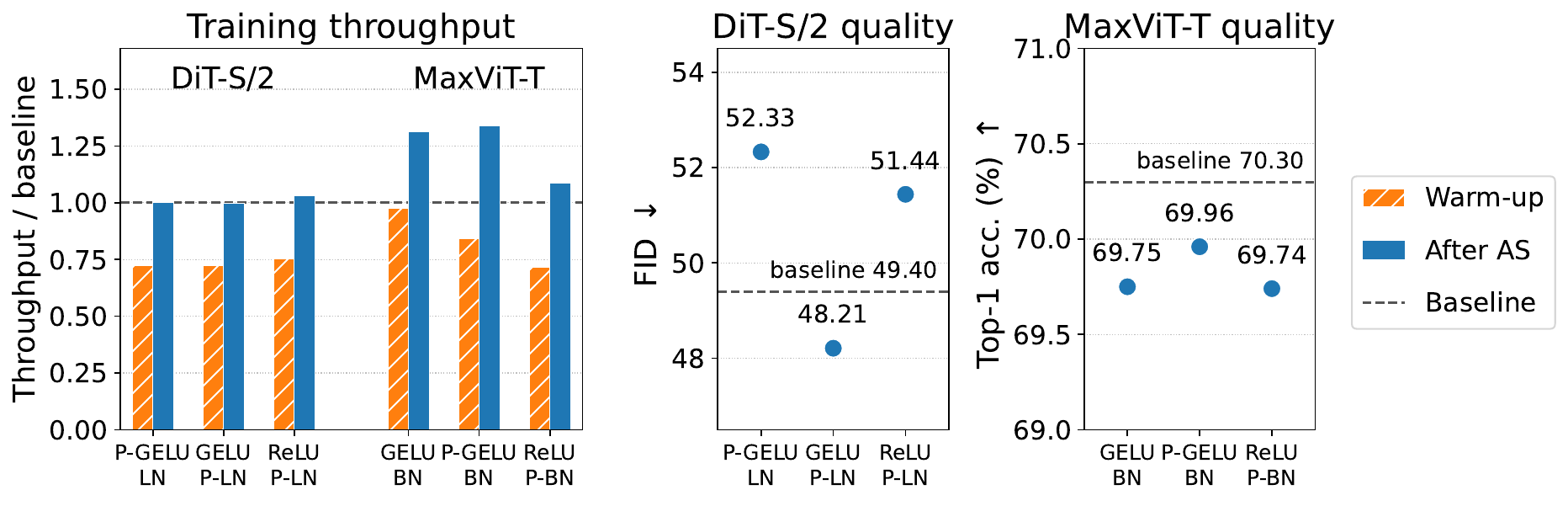}
    \caption{Accumulation Stop on DiT-S/2 and MaxViT-T. \textbf{Left:} training throughput
    normalized to the baseline with PyTorch's optimized normalization kernels (dashed line). \textbf{Middle and Right:} final
    FID of DiT-S/2 and accuracy of MaxViT-T with AS, against the baseline. P- denotes
    percentile shifting. Warm-up is slower than the baseline because we use naive PyTorch implementation.
   Full results are given in Appendix Tables~\ref{tab:throughput}
    and~\ref{tab:quality}.}
    \label{fig:accumulation_stop}
\end{figure*}


\vspace{-0.5em}
\section{Limitations and Discussion}
\label{sec:limitations}
\vspace{-0.5em}
We summarize the main limitations of our work below. Each is discussed in detail
in Appendix~\ref{app:discussion}.
\textbf{(1)} Our theoretical results hold strictly only at initialization.
Appendix~\ref{app:beyond_init} extends the argument to later training informally.
\textbf{(2)} The formal proof covers ReLU only, and its extension to smooth activations
is supported empirically.
\textbf{(3)} Language modeling is evaluated only with GPT-nano (124M parameters) on FineWeb,
and clipping thresholds for squared activations are selected empirically and appear
architecture-dependent.
\textbf{(4)} Accumulation Stop is evaluated only on architectures with fixed-size inputs,
and its throughput is measured with a naive PyTorch implementation.
\textbf{(5)} Our analysis does not explain the strong robustness of transformers
(ViT, GPT-nano) to aggressive sparsification.
\textbf{(6)} We do not explain why Percentile LayerNorm with GELU improves FID and IS
on DiT-S/2.

\section{Conclusion}
We described negative weight drift, a systematic shift of weights toward negative
values caused by the interaction between standard losses and positively biased
activation functions. The drift emerges in the first iterations of training, persists
even on random data, and appears across architectures that lack mean-centering
normalization. Because the drift pushes pre-activations below zero, it determines how
many activations are zeroed or suppressed, and it should therefore be accounted for when designing 
new architectures, especially as modern models
replace mean-centering normalization with RMSNorm, which does not counteract it.

The sparsity induced by drift is a feature rather than a bug, since performance is
nearly unaffected up to ${\sim}70\%$ post-activation sparsity. Sparsity levels can be
controlled through the choice of activation function and post-training procedures such
as ReLUfication. However, some activation functions, such as ReLU$^2$, may amplify large
activation values, which may degrade performance. Both weight drift and spikes depend on normalization,
since mean-centering suppresses drift and centering or percentile shifts reduce spike
magnitude. 

The common theme of these findings is the activations themselves. While this paper mainly
focuses on weight drift and its effect on sparsity, these results show that drift,
activation functions and normalization should be considered together, as they jointly
shape the activations of a network and contribute to overall performance.

\section{Related Work}
\label{sec:related_work}

\subsection{Weight drift and internal covariate shift}

The concept of \emph{internal covariate shift} was introduced by
\citet{ioffe2015batch} to describe the continuously changing
distribution of layer inputs during training, as upstream weights
update, the activations they produce shift, forcing downstream
layers to perpetually adapt to a moving target. Batch Normalization
(BN) was proposed as a remedy by normalizing layer inputs to zero
mean and unit variance, with subsequent work establishing additional
benefits beyond distributional stability~\citep{santurkar2018does}. Our work touches the same phenomenon but from a different angle.
Rather than studying how changing activations force downstream
weights to adapt, we study how positively biased activation
functions \emph{cause} weights to drift systematically negative.
The most important connection is that while mean-centering
normalization was developed to overcome covariate shift, it also
incidentally suppresses weight drift by removing the positive
pre-activation bias that drives it. However, modern architectures
increasingly adopt normalization layers without centering 
such as RMSNorm~\citep{zhang2019root}  meaning weight drift
is likely to occur in these settings.

\subsection{Normalization layers and centering in modern architectures}
\label{app:normalizations}
Most contemporary large language models replace LayerNorm~\citep{ba2016layer} with
RMSNorm~\citep{zhang2019root}, which rescales activations by their root-mean-square
but does not subtract the mean. This includes the LLaMA family~\citep{touvron2023llama}
 and most other open-weight LLMs released after 2023 (Mistral, Gemma, Qwen, DeepSeek).
The original motivation was efficiency, \citet{zhang2019root} reported that the centering
step contributed little to model quality while accounting for noticeable computational
cost. Vision and generative-image architectures present a more heterogeneous picture.
Classical ViTs~\citep{dosovitskiy2021vit} retain LayerNorm, and U-Net diffusion
backbones rely on GroupNorm, both of which center activations. The original Diffusion
Transformer~\citep{peebles2023dit} and its successors use adaptive LayerNorm (adaLN)
for timestep and class conditioning, again preserving centering on the main residual
path. Moreover, recent large-scale diffusion transformers have begun adopting RMSNorm
in the attention pathway. Stable Diffusion~3 and SD~3.5~\citep{esser2024sd3} and
the FLUX family~\citep{blackforestlabs2024flux} apply RMSNorm to query and key
projections to stabilize mixed-precision training, while retaining adaLN elsewhere.

\subsection{Emergent activation sparsity}
\citeauthor{li2022lazy} documented that activation sparsity emerges 
in trained transformers without explicit regularization and 
persists on random data, implicating the optimizer rather than 
the data distribution as the causal factor. Our work can be seen 
as both an extension and a deepening of this finding.  
We show that weight drift is 
governed by the local 
interaction between the activation function, the normalization 
layer, and the loss function. We identify \emph{negative weight drift} as the precise 
mechanistic cause of activation sparsity. On the 
theoretical side, our proof of gradient positivity extends that 
of \citet{li2022lazy} in two directions, we establish the result 
for \emph{any} intermediate layer rather than only the final 
two, and we operate under strictly weaker assumptions by 
permitting non-negative off-diagonal weight correlations rather 
than requiring them to be zero. Empirically, we go beyond characterizing sparsity as a 
phenomenon by evaluating whether naturally arising sparsity 
hurts model performance or not.

\subsection{Practical side of sparse activations}
Pruning is most effective when the values to be pruned are 
already near zero, which is naturally the case for the 
intermediate representations of many LLMs~\citep{liu2023deja, 
li2022lazy}. This has motivated a line of work exploiting 
activation sparsity to accelerate the decoding stage via 
faster \emph{sparse vector}--\emph{dense matrix} 
multiplications, achieving up to $2\times$ speedup using 
specialized kernels~\citep{song2024turbo, song2024powerinfer, 
liu2024training, lee2024cats}. These gains are most pronounced 
at batch size 1 and diminish as batch size 
increases~\citep{shrestha2025polar}, since the speedup arises 
by skipping rows of the weight matrix that correspond to zero 
activations, an advantage that erodes under batched 
computation. Many of these methods additionally require a 
predictive mechanism to prefetch the relevant weight indices 
into memory ahead of time~\citep{liu2023deja}. Our work 
complements this line by identifying the optimization dynamics 
responsible for producing the sparsity these methods depend on, 
and by characterizing the sparsity--accuracy tradeoff that 
determines how aggressively it can be exploited without 
degrading model quality.

\subsection{Activation spikes at intermediate layers}
\citeauthor{sun2026spike} studied massive activation spikes in 
SwiGLU-based transformers, attributing them to a directional 
quadratic amplification mechanism arising from the interaction 
between gate and up projections. We independently observed 
qualitatively similar spikes in standard non-gated MLP blocks, 
a finding we arrived at separately while studying the effects 
of squared nonlinearities on GPT-nano. The fact that spikes 
arise in both gated and non-gated architectures suggests that 
intermediate-layer spiking is a general property of transformer language models
rather than a SwiGLU-specific artifact as suggested by \citet{sun2026spike}. This is further 
supported by \citet{yusupov2025geometric}, who observe 
consistent early-layer spikes in geometric properties of 
internal representations including Maximum Explainable 
Variance, Effective Rank, and Intrinsic Dimensionality 
across multiple models. We contribute 
three further observations: (1) normalization reduces but does not 
eliminate the spikes, (2) spikes originate in 
the up-projection of the MLP block, and (3) controlling spike amplitude stabilizes training.
Regarding the origin of spikes, \citet{barbero2025first} 
argue that LLMs attend to the first token to avoid over-mixing of token representations, 
so the first or ``special'' tokens acquire a functional role. 
This offers a possible explanation for why spikes concentrate 
on specific tokens in early layers. 
Our profiling of pretrained models presented in Appendix~\ref{app:pretrained_spikes}
 is consistent with this view. We observe that
  distinct early-layer spikes appear in GPT-2 and BERT but not in vision models. 

\bibliography{main}
\bibliographystyle{unsrtnat}

\clearpage
\appendix

\section*{Appendix}
\label{sec:appendix}
\makeatletter
\gdef\addcontentsline#1#2#3{%
  \addtocontents{#1}{\protect\contentsline{#2}{#3}{\thepage}{}}%
}
\makeatother


\vspace{0.5ex}
{
  \small 
  \setlength{\parskip}{0pt}
  \etocsetnexttocdepth{subsubsection}
  
  \etocsettocstyle{}{} 
  
  \localtableofcontents
}
\vspace{2ex}


\section{Theory}
\label{app:theory}

\subsection{Theorem \& proof: positive expected gradient for MSE loss}
\label{sec:mse_proof}
\begin{theorem}[MSE loss]
Let $f(\mathbf{x}) = \mathbf{V}_{\mathrm{eff}}^{(l)}\,\sigma(\mathbf{p}^{(l)})$
be as in~\eqref{eq:folded}, with $\sigma = \mathrm{ReLU}$ and
$\mathbf{V}_{\mathrm{eff}}^{(l)}$ satisfying
Theorem~\ref{th:veff_properties}. Assume the network is at initialization,
so that $\mathbf{V}_{\mathrm{eff}}^{(l)}$ is independent of
$\mathbf{p}^{(l)}$ and of the target $\mathbf{y}$. Consider the MSE loss
$\ell(f(\mathbf{x}), \mathbf{y}) = \tfrac{1}{2}\|f(\mathbf{x}) - \mathbf{y}\|_2^2$.
Then for any neuron $i$,
\begin{equation}
    \mathbb{E}\!\left[\frac{\partial \ell}{\partial p_i^{(l)}}\right] \geq 0,
\end{equation}
with strict inequality whenever $p_i^{(l)} > 0$, where the expectation
is taken over $\mathbf{V}_{\mathrm{eff}}^{(l)}$.
\end{theorem}

\begin{proof}
Let $i, j \in \{1, \dots, d_p\}$ index the neurons of
layer $l$, and let $\mathbf{v}_j$ denote the $j$-th column of
$\mathbf{V}_{\mathrm{eff}}^{(l)}$. If $p_i^{(l)} \leq 0$ then
$\sigma'(p_i) = 0$ and $\partial \ell / \partial p_i = 0$, so the bound
holds. We assume $p_i^{(l)} > 0$ for the remainder, in which
case $\sigma'(p_i) = 1$. Writing the matrix–vector product
columnwise gives
$f(\mathbf{x}) = \sum_{j=1}^{d_p} \mathbf{v}_j \, \sigma(p_j)$, and only
the $j = i$ term depends on $p_i$, so
$\partial f / \partial p_i = \mathbf{v}_i \, \sigma'(p_i)$. Applying the
chain rule to the MSE loss and substituting,
\begin{equation}
    \frac{\partial\ell}{\partial p_i}
    = \sigma'(p_i) \sum_{j=1}^{d_p}
        \langle\mathbf{v}_i,\mathbf{v}_j\rangle\,\sigma(p_j)
      \;-\; \sigma'(p_i)\,\langle\mathbf{v}_i,\mathbf{y}\rangle.
    \label{eq:grad_expanded}
\end{equation}

By taking the expectation over $\mathbf{V}_{\mathrm{eff}}^{(l)}$, since $\mathbf{y}$ and the $p_j$ are independent
of $\mathbf{V}_{\mathrm{eff}}^{(l)}$ at initialization. The label term therefore vanishes,
\begin{equation}
    \mathbb{E}\!\left[\langle \mathbf{v}_i, \mathbf{y}\rangle\right]
    = \langle \mathbb{E}[\mathbf{v}_i], \mathbf{y}\rangle = 0,
\end{equation}
using $\mathbb{E}[\mathbf{v}_i] = \mathbf{0}$ from
Theorem~\ref{th:veff_properties}, leaving
\begin{equation}
    \mathbb{E}\!\left[\frac{\partial\ell}{\partial p_i}\right]
    = \sum_{j=1}^{d_p}
        \mathbb{E}\!\left[\langle\mathbf{v}_i,\mathbf{v}_j\rangle\right]\,
        \sigma(p_j).
\end{equation}

Each $\sigma(p_j) \geq 0$ since
$\sigma = \mathrm{ReLU}$, and
$\mathbb{E}[\langle\mathbf{v}_i,\mathbf{v}_j\rangle] \geq 0$ by
Theorem~\ref{th:veff_properties}, so every summand is non-negative. The
$j = i$ term is strictly positive, since $\mathbb{E}[\|\mathbf{v}_i\|^2] > 0$
and $\sigma(p_i) = p_i > 0$. Thus
$\mathbb{E}[\partial\ell/\partial p_i] \geq 0$, with strict inequality
whenever $p_i^{(l)} > 0$.
\end{proof}

\begin{remark}
When averaging over inputs, a covariance term 
$\mathrm{Cov}_{\mathbf{x}}(\partial\ell/\partial p_i,\; 
x_k^{(l-1)})$ arises. However, we further empirically demonstrate that this term oscillates around zero at early iterations (Figure~\ref{fig:mlp_ce_cov}). 
Here, $\mathbf{x}$ denotes the input data vector (or layer input $\mathbf{x}^{(l-1)}$), while $\mathbf{p}^{(l)}$ is the pre-activation at layer $l$, given by $\mathbf{p}^{(l)} = \mathbf{W}^{(l)} \mathbf{x}^{(l-1)} + \mathbf{b}^{(l)}$. 
The covariance is taken over the empirical/data distribution of $\mathbf{x}$.
\end{remark}

\subsection{Theorem \& proof: positive expected gradient for cross-entropy loss}
\label{sec:ce_proof}

\begin{theorem}[Cross-entropy loss]
Let $f(\mathbf{x}) = \mathbf{V}_{\mathrm{eff}}^{(l)}\,\sigma(\mathbf{p}^{(l)})$
be as in~\eqref{eq:folded}, with $\sigma = \mathrm{ReLU}$ and
$\mathbf{V}_{\mathrm{eff}}^{(l)} \in \mathbb{R}^{C \times d_p}$ satisfying
Theorem~\ref{th:veff_properties}, where $C$ denotes the number of output
classes. Assume the network is at initialization, so that
$\mathbf{V}_{\mathrm{eff}}^{(l)}$ is independent of $\mathbf{p}^{(l)}$ and
of the one-hot label $\mathbf{y}$. Consider the cross-entropy loss
$\ell(f(\mathbf{x}), \mathbf{y}) = -\sum_{c=1}^{C} y_c \log s_c$,
where $\mathbf{s} = \mathrm{softmax}(f(\mathbf{x}))$. Then for any
neuron $i$,
\begin{equation}
    \mathbb{E}\!\left[\frac{\partial \ell}{\partial p_i^{(l)}}\right] \geq 0,
\end{equation}
(up to higher-order
corrections from the softmax linearization), with strict inequality whenever $p_i^{(l)} > 0$, where the expectation
is taken over $\mathbf{V}_{\mathrm{eff}}^{(l)}$.
\end{theorem}

\begin{proof}
If $p_i^{(l)} \leq 0$ then $\sigma'(p_i) = 0$ and the bound holds trivially,
so assume $p_i^{(l)} > 0$ and $\sigma'(p_i) = 1$. Define
$\mathbf{s} = \mathrm{softmax}(f(\mathbf{x}))$ as the predicted class
distribution. The softmax–cross-entropy gradient with respect to the
logits is $\partial \ell / \partial f_c = s_c - y_c$. Using the column-wise expansion $f(\mathbf{x}) = \sum_{j=1}^{d_p} \mathbf{v}_j\,\sigma(p_j)$,
only the $j = i$ term depends on $p_i$, so
$\partial f / \partial p_i = \mathbf{v}_i$, and the chain rule gives
\begin{equation}
    \frac{\partial \ell}{\partial p_i}
    = \langle \mathbf{v}_i, \mathbf{s}\rangle - \langle \mathbf{v}_i, \mathbf{y}\rangle.
    \label{eq:ce_grad}
\end{equation}

\textbf{Label term.} Since $\mathbf{y}$ is independent of
$\mathbf{V}_{\mathrm{eff}}^{(l)}$ at initialization,
\begin{equation}
    \mathbb{E}\!\left[\langle\mathbf{v}_i, \mathbf{y}\rangle\right]
    = \langle \mathbb{E}[\mathbf{v}_i], \mathbf{y}\rangle = 0,
\end{equation}
by Theorem~\ref{th:veff_properties}, exactly as in the MSE case.

\textbf{Softmax term.} The simplex constraint $\mathbf{1}^\top \mathbf{s} = 1$
means $\mathbf{s}$ is not zero-mean, so we cannot apply
Theorem~\ref{th:veff_properties} directly. Decompose
$\mathbf{s} = \tfrac{1}{C}\mathbf{1} + \tilde{\mathbf{s}}$, where
$\tilde{\mathbf{s}} = \mathbf{P}\mathbf{s}$ is the centered softmax and
$\mathbf{P} = \mathbf{I} - \tfrac{1}{C}\mathbf{1}\mathbf{1}^\top$ is the
centering projection. The constant component vanishes:
$\mathbb{E}[\langle\mathbf{v}_i, \tfrac{1}{C}\mathbf{1}\rangle] = 0$,
leaving
\begin{equation}
    \mathbb{E}\!\left[\frac{\partial \ell}{\partial p_i}\right]
    = \mathbb{E}\!\left[\langle \mathbf{v}_i, \tilde{\mathbf{s}}\rangle\right].
\end{equation}

We linearize the softmax at the origin. Using
$\mathbf{J}_{\mathbf{s}}(\mathbf{0}) = \tfrac{1}{C}\,\mathbf{P}$,
\begin{equation}
    \tilde{\mathbf{s}} = \tfrac{1}{C}\,\mathbf{P}\mathbf{f}
        + O(\|\mathbf{f}\|^2),
    \label{eq:softmax_linear}
\end{equation}
and substituting $\mathbf{f} = \sum_j \mathbf{v}_j\,\sigma(p_j)$,
\begin{equation}
    \mathbb{E}\!\left[\langle \mathbf{v}_i, \tilde{\mathbf{s}}\rangle\right]
    = \frac{1}{C}\sum_{j=1}^{d_p} \sigma(p_j)\,
      \mathbb{E}\!\left[\mathbf{v}_i^\top \mathbf{P}\,\mathbf{v}_j\right]
      \;+\; O(\|\mathbf{f}\|^2).
    \label{eq:ce_main}
\end{equation}

\textbf{Evaluating the projected inner product.}
Expanding $\mathbf{P} = \mathbf{I} - \tfrac{1}{C}\mathbf{1}\mathbf{1}^\top$ gives
\begin{equation}
    \mathbb{E}\!\left[\mathbf{v}_i^\top \mathbf{P}\, \mathbf{v}_j\right]
    = \mathbb{E}\!\left[\langle \mathbf{v}_i, \mathbf{v}_j\rangle\right]
        - \tfrac{1}{C}\,\mathbb{E}\!\left[(\mathbf{v}_i^\top\mathbf{1})(\mathbf{v}_j^\top\mathbf{1})\right],
    \label{eq:proj_split}
\end{equation}
so it suffices to evaluate the second term. Write
$\mathbf{V}_{\mathrm{eff}}^{(l)} = \mathbf{W}_L\,M$,  where
$M = \mathbf{D}_{L-1}\mathbf{W}_{L-1}\cdots\mathbf{W}_{l+1}$ is independent
of $\mathbf{W}_L$ at initialization. The rows of $\mathbf{W}_L$ are i.i.d.\
zero-mean, so for any $c \neq c'$ and any $i, j$,
\begin{multline}
    \mathbb{E}\!\left[(\mathbf{v}_i)_c\,(\mathbf{v}_j)_{c'}\right] \\
    = \mathbb{E}_M\!\left[\mathbb{E}_{\mathbf{W}_L}\!\left[
        (\mathbf{W}_L)_{c,:} M_{:,i} \cdot (\mathbf{W}_L)_{c',:} M_{:,j}\,\big|\,M
      \right]\right] = 0,
    \label{eq:row_decorrelation}
\end{multline}
using independence and zero-mean of rows of $\mathbf{W}_L$. By the same
i.i.d.\ row structure, $\mathbb{E}[(\mathbf{v}_i)_c (\mathbf{v}_j)_c]$ is
constant in $c$, hence equals
$\tfrac{1}{C}\,\mathbb{E}[\langle\mathbf{v}_i, \mathbf{v}_j\rangle]$.
Therefore
\begin{multline}
    \mathbb{E}\!\left[(\mathbf{v}_i^\top\mathbf{1})(\mathbf{v}_j^\top\mathbf{1})\right]
    =  \sum_{c,c'} \mathbb{E}[(\mathbf{v}_i)_c(\mathbf{v}_j)_{c'}]
    = \\  \sum_c \mathbb{E}[(\mathbf{v}_i)_c(\mathbf{v}_j)_c]
    = \mathbb{E}[\langle\mathbf{v}_i, \mathbf{v}_j\rangle],
\end{multline}
and consequently
\begin{multline}
    \mathbb{E}\!\left[\mathbf{v}_i^\top \mathbf{P}\,\mathbf{v}_j\right]
    = \mathbb{E}[\langle\mathbf{v}_i, \mathbf{v}_j\rangle]
        - \tfrac{1}{C}\,\mathbb{E}\!\left[(\mathbf{v}_i^\top\mathbf{1})(\mathbf{v}_j^\top\mathbf{1})\right]
    = \\ \tfrac{C-1}{C}\,\mathbb{E}[\langle\mathbf{v}_i, \mathbf{v}_j\rangle]
    \;\geq\; 0,
    \label{eq:proj_corr}
\end{multline}
by Theorem~\ref{th:veff_properties}.

\textbf{Sign analysis.} Substituting~\eqref{eq:proj_corr}
into~\eqref{eq:ce_main},
\begin{equation}
    \mathbb{E}\!\left[\frac{\partial \ell}{\partial p_i}\right]
    = \frac{C-1}{C^2}\sum_{j=1}^{d_p}
        \mathbb{E}[\langle\mathbf{v}_i, \mathbf{v}_j\rangle]\,\sigma(p_j)
      \;+\; O(\|\mathbf{f}\|^2).
\end{equation}
To leading order, every summand is non-negative
(by Theorem~\ref{th:veff_properties} and ReLU non-negativity)
$\tfrac{C-1}{C^2}\,\mathbb{E}[\|\mathbf{v}_i\|^2]\,p_i^{(l)} \geq 0$.
\end{proof}

\begin{remark}
The CE bound parallels Theorem~\ref{th:pos_p_gradient} with two structural
differences. First, the centering projection $\mathbf{P}$ replaces the
identity, contributing the factor $(C-1)/C$ that reflects the simplex
constraint $\mathbf{1}^\top \mathbf{s} = 1$. Second, the overall
coefficient $1/C^2$ reflects the softmax derivative scale at the origin.
The leading non-vanishing correction to~\eqref{eq:softmax_linear} that
survives expectation is $O(\|\mathbf{f}\|^4)$: the cubic term in
$\mathbf{V}_{\mathrm{eff}}$ has zero expectation by sign-flip symmetry of
$\mathbf{W}_L$ (which leaves the gates $\mathbf{D}_l, \dots, \mathbf{D}_{L-1}$
invariant since they depend only on $\mathbf{W}_1, \dots, \mathbf{W}_{L-1}$).
The same negative-drift consequence then follows: the gradient with
respect to weights $W_{i,k}^{(l)}$ factorizes as
$(\partial\ell/\partial p_i)\,\sigma(p_k^{(l-1)})$, both factors
non-negative in expectation, so the expected weight update is
non-positive.
\end{remark}

\subsection{Extension to skip connections}
\label{app:skip_theory}

Theorem~\ref{th:veff_properties} does not use the fact that the layers after $l$ form a single chain. 
It requires only that the map from $\sigma(\mathbf{p}^{(l)})$ 
to the output is built from zero-mean weights and ReLU gates and ends in the network's last weight matrix $\mathbf{W}_L$. 
A skip connection adds parallel routes without changing either property, 
so $\mathbb{E}[\mathbf{v}_i] = \mathbf{0}$ and $\mathbb{E}[\langle \mathbf{v}_i, \mathbf{v}_j \rangle] \geq 0$ 
carry over unchanged. The part of the output that bypasses layer $l$, which we denote $\mathbf{r}$, 
does not depend on $\mathbf{p}^{(l)}$ and simply shifts the effective target from $\mathbf{y}$ to $\mathbf{y} - \mathbf{r}$.
 It is not independent of $\mathbf{v}_i$, 
 since both pass through $\mathbf{W}_L$, but the resulting inner product $\langle \mathbf{v}_i, \mathbf{r} \rangle$
  has the same form as the cross-correlation term that the theorem already bounds.

\paragraph{Empirical support.}
Skip connections do not change the structure of the argument, and we verify this empirically. 
Figure~\ref{fig:mlp_training_dynamics} compares an MLP with and without a skip connection, and the
 drift trajectory is qualitatively unchanged. ResNet-18 (Figure~\ref{fig:resnet_drift}) and MaxViT 
 are skip-connection architectures and both exhibit the predicted positive gradient bias, and the MLP with
  a skip connection appears throughout Tables~\ref{tab:pruning_results_resized} and~\ref{tab:pruning_results_extended}. 
  However, skip connections change the magnitude of the effect. In Figure~\ref{fig:mlp_training_dynamics}, 
  adding a skip connection reduces ReLU sparsity from ${\sim}0.6$ to ${\sim}0.4$ while increasing gradient magnitudes, 
  and Tables~\ref{tab:pruning_results_resized} and~\ref{tab:pruning_results_extended} show that skip connections substantially 
  improve tolerance to imposed sparsity.

\subsection{Weight drift beyond initialization}
\label{app:beyond_init}

Our theorems hold strictly at initialization. At $t > 0$, the weights are no longer zero-mean and are
no longer independent of each other or of the pre-activations. The main insight is that the gradient
magnitude decreases once the model predicts correctly, while most of the drift happens while it still
predicts incorrectly.

Assuming the covariance term stays orders of magnitude smaller than the weight mean, as we measure in
Figure~\ref{fig:mlp_ce_cov}, we write the weights after $t$ steps as
$\mathbf{W}_t = \mathbf{W}_0 + \boldsymbol{\Delta}_t$. The analysis at initialization still applies 
approximately as long as the weights stay close to their
initial values, $\|\boldsymbol{\Delta}_t\| \ll \sigma(\mathbf{W}_0)$. This ratio is the Z-score in
Figure~\ref{fig:z_score_optimizers}.
 The term at risk is then the label term
$-\langle \mathbf{v}_i, \mathbf{y} \rangle$ in Eq.~\eqref{eq:grad_expanded}. Once downstream weights
align with the labels, it reduces the positive gradient. However, gradients are largest early in training, 
when the loss is large, so most of the weight shift happens in the first
iterations. By the time predictions align with the labels, gradient magnitudes have decayed and the
weights can no longer change substantially. This is consistent with the knee and plateau we observe in
the drift trajectories (\S\ref{sec:weight_drift}).
\emph{The drift could only be avoided if the zero-mean initialization already lay close to a good
solution, which is not the case in practice.}

\section{Extended Results}
\label{app:extended_results}

\subsection{Additional experiments on weight drift}

\label{app:additional_drift}

This section complements the empirical analysis of 
Section~\ref{sec:weight_drift} with three sets of additional 
experiments: (1) the effect of Top-K sparsity on gradient bias and 
weight drift, (2) cross-architecture validation of the drift phenomenon,
and (3) the stability of drift dynamics when normalization statistics 
are frozen via \emph{Accumulation Stop}.

\subsubsection{Effect of Top-K sparsity on weight drift}
Figure~\ref{fig:topk_drift} shows gradient mean and weight mean 
trajectories for a three-layer MLP with GELU activations trained 
under varying levels of Top-K sparsity on random data. As sparsity decreases or $K$ 
increases and more activations are retained, the positive gradient 
bias and corresponding negative weight drift become progressively 
more pronounced. At $K=0.10$, where only 10\% of 
activations are retained the gradient mean 
collapses toward zero across all layers and weight means remain 
nearly flat throughout training. This indicates that extreme 
sparsity starves the network of informative gradient signal 
too few activations survive to produce consistent updates 
resulting in severely slow convergence. This observation suggests 
a practical lower bound on usable sparsity.

\begin{figure*}[ht]
    \centering
    \includegraphics[width=1\linewidth]{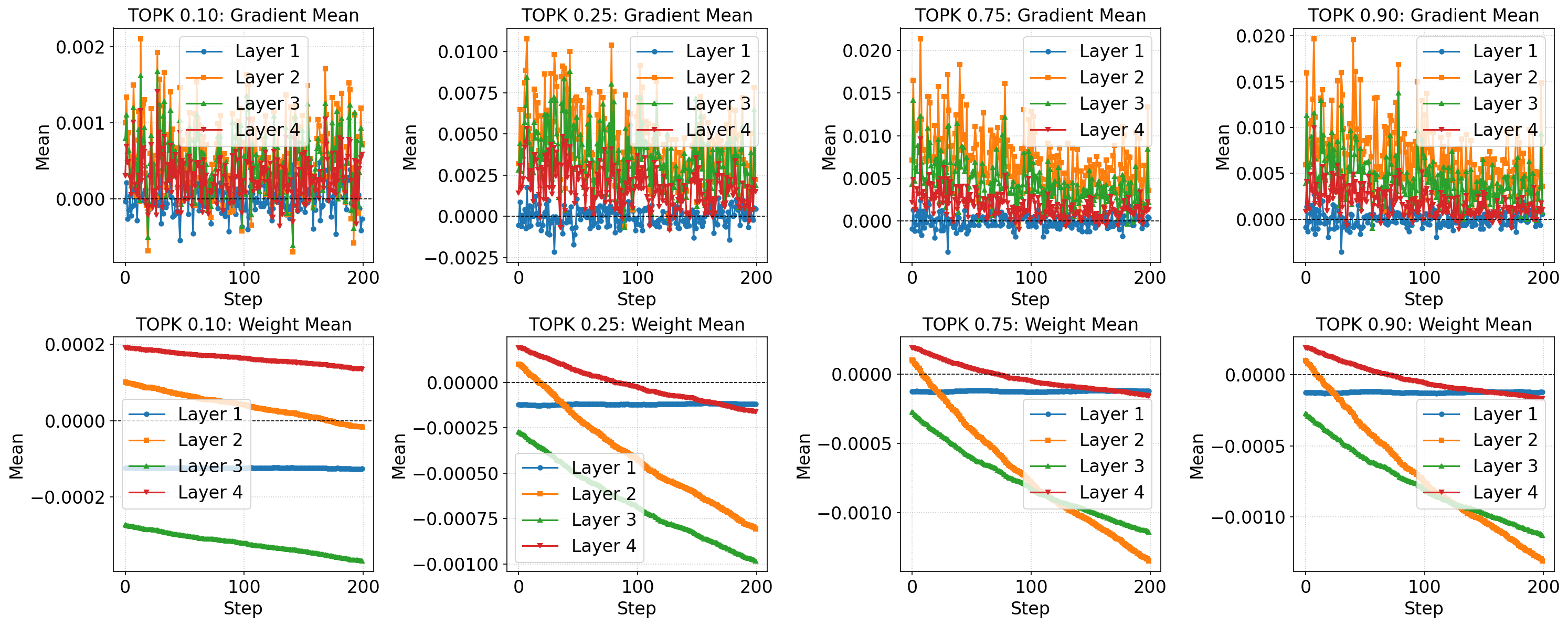}
    \caption{
        \textbf{Effect of Top-K sparsity on weight drift.}
        A three-layer MLP with GELU activations is trained on 
        random $\{X, Y\}$ pairs sampled from $\mathcal{N}(0,1)$ 
        under four Top-K retention levels ($K \in 
        \{0.10, 0.25, 0.75, 0.90\}$). Top row: gradient mean per 
        layer. Bottom row: weight mean per layer. Trajectories are 
        averaged across 20 runs (hidden dimension 128, SGD with 
        lr$=0.001$). As $K$ increases, gradient bias and weight 
        drift grow more pronounced. At $K=0.10$, gradient signal 
        collapses and weight means remain flat, indicating that 
        extreme sparsity prevents consistent weight updates.
    }
    \label{fig:topk_drift}
\end{figure*}

\begin{figure*}[t]
    \centering
    \includegraphics[width=0.95\linewidth]{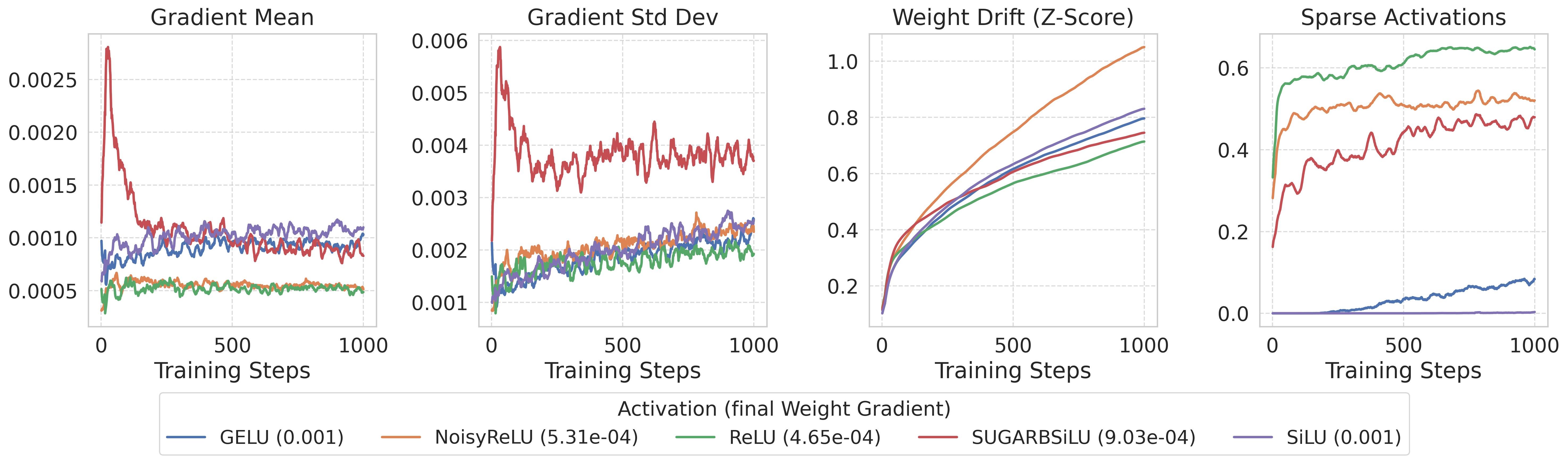}\\[6pt]
    \includegraphics[width=0.95\linewidth]{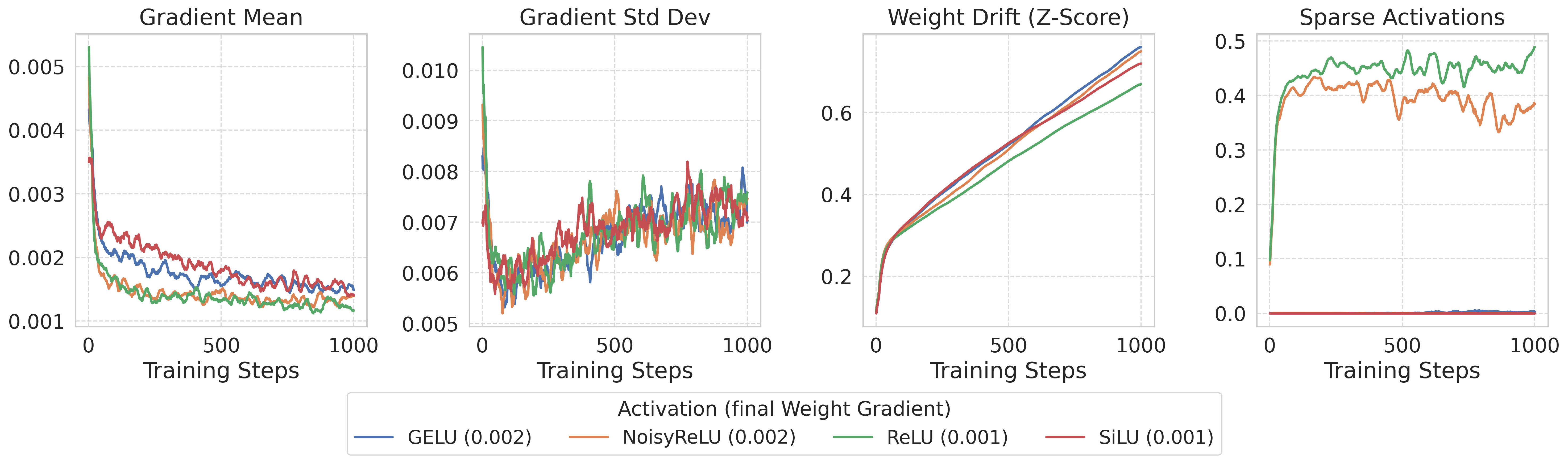}
    \caption{
        \textbf{Training dynamics of an MLP without (top) and 
        with (bottom) skip connection} 
        (see Appendix~\ref{sec:mlp_details} for architecture 
        details), under GELU, NoisyReLU, ReLU, SiLU, and 
        SUGARBSiLU with Adam optimizer (lr$=10^{-3}$). 
        Four metrics are tracked over 1000 training steps: 
        (1) average gradient magnitude, (2) gradient standard 
        deviation, (3) weight drift (Z-score), and (4) fraction 
        of sparse activations. In the plain MLP, ReLU and 
        NoisyReLU induce the highest sparsity (${\sim}60\%$), and 
        NoisyReLU shows the largest drift; GELU and SiLU remain near-zero 
        sparsity. Adding a skip connection reduces sparsity to 
        ${\sim}40\%$ for ReLU and NoisyReLU while substantially 
        increasing gradient magnitudes across all activations. 
        Final average gradient magnitudes are reported in 
        parentheses in the legend.
    }
    \label{fig:mlp_training_dynamics}
\end{figure*}

\subsubsection{Cross-architecture validation}
Figures~\ref{fig:resnet_drift} 
and~\ref{fig:senet} confirm that negative weight drift is not 
an artifact of the MLP setting but arises consistently across 
diverse architectures and optimizers. In ResNet-18 
(Figure~\ref{fig:resnet_drift}), trained with Adam at 
lr$=10^{-3}$, weight drift accumulates monotonically across 
GELU, ReLU, and SiLU, with the covariance term 
$\mathrm{Cov}(\partial\ell/\partial p, x)$ remaining orders 
of magnitude smaller than the weight mean. 

MaxViT-Tiny 
\citep{tu2022maxvit} presented in Figure~\ref{fig:maxvit}, trained with 
AdamW at lr$=3\times10^{-3}$, shows a positive gradient bias 
for all three activations. ReLU accumulates negative weight drift, 
while GELU and SiLU drift slightly positive, possibly because of 
interactions with BatchNorm in its MBConv blocks.

Taken together, these results confirm that negative weight 
drift is a robust and broadly observed phenomenon, while 
acknowledging that architecture-specific factors can influence 
its magnitude and direction in certain settings.

\begin{figure*}[ht]
    \centering
    \includegraphics[width=\linewidth]{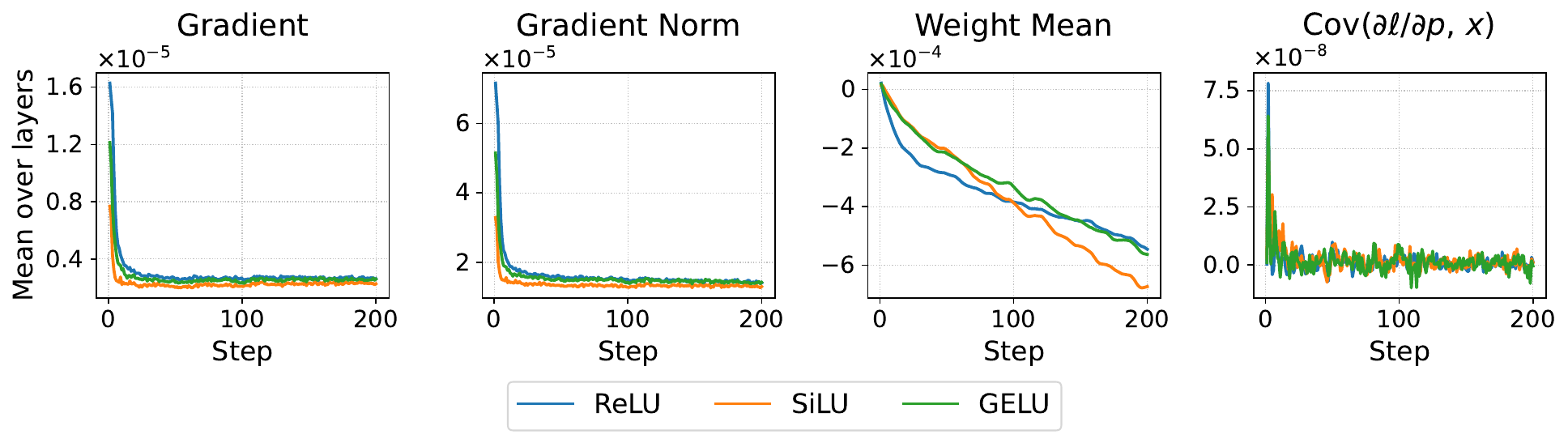}
    \caption{
        \textbf{Weight drift in ResNet-18.}
        Training dynamics under GELU, ReLU, and SiLU with Adam 
        optimizer (lr$=10^{-3}$). Four metrics are tracked: 
        (1) average gradient mean, (2) gradient norm, 
        (3) weight mean, and (4) covariance term 
        $\mathrm{Cov}(\partial\ell/\partial p,\, x)$.
    }
    \label{fig:resnet_drift}
\end{figure*}

\begin{figure*}[ht]
    \centering
    \includegraphics[width=\linewidth]{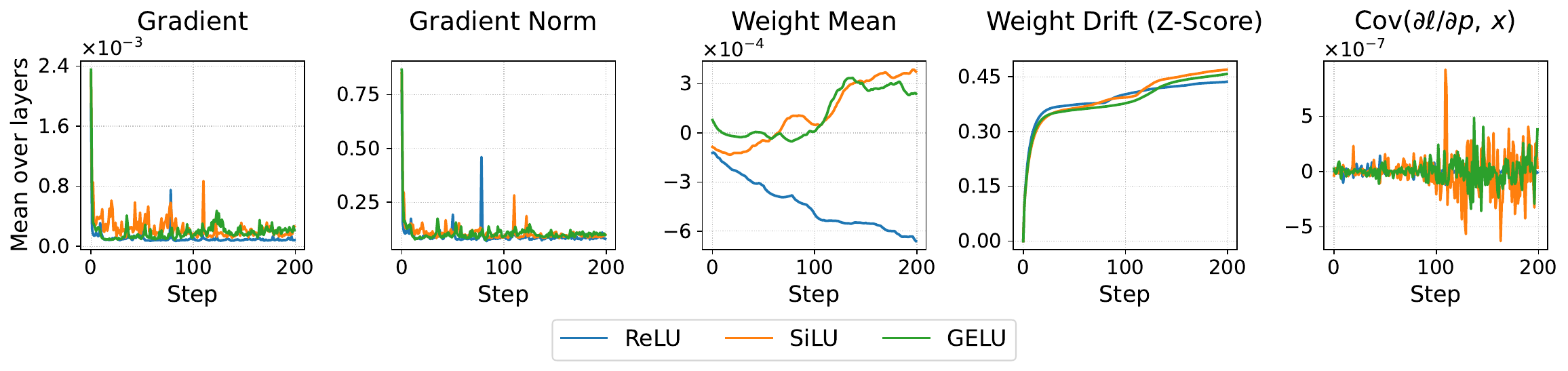}
    \caption{
        \textbf{Weight drift in MaxViT-Tiny.}
        Training dynamics under GELU, ReLU, and SiLU with AdamW 
        optimizer (lr$=3\times10^{-3}$). All three activations 
        exhibit a positive gradient bias. ReLU drifts negative, 
        while GELU and SiLU drift slightly positive.
    }
    \label{fig:maxvit}
\end{figure*}


\subsubsection{MLP with and without skip connections}
Figure~\ref{fig:mlp_training_dynamics} tracks four metrics over 
1000 training steps of an MLP trained on real data with Adam 
(lr$=10^{-3}$), comparing GELU, SiLU, ReLU, NoisyReLU, and 
SUGARBSiLU, both without (top) and with (bottom) a skip 
connection. No normalization layers were used in either 
configuration, so all observed dynamics arise purely from the 
interaction between the activation function and the optimizer.

\textit{Gradient magnitude and variance.}
All activations exhibit a sharp transient in both gradient 
magnitude and standard deviation during the first ${\sim}100$ 
steps, followed by rapid decay to a stable regime. NoisyReLU 
settles at a notably higher gradient variance than the other 
activations, reflecting the stochastic nature of its outputs, 
SUGARBSiLU exhibits persistently noisy gradients throughout 
training, which may hinder stable convergence. Adding a skip 
connection increases gradient magnitude across all activations, 
indicating that the residual path sustains stronger gradient 
flow to earlier layers.

\textit{Weight drift.}
The Z-score of weight magnitudes increases monotonically for 
all activations, confirming that negative weight drift 
accumulates continuously on real data, consistent with the 
random-data experiments of \S\ref{sec:weight_drift}. 
The drift trajectory exhibits a knee at approximately the same 
step as the gradient transient, drift accumulates rapidly in 
early training and then slows as the gradient bias diminishes. 
In the plain MLP, NoisyReLU produces the most pronounced drift, 
while ReLU, GELU, SiLU, and SUGARBSiLU reach similar levels. The skip connection does not qualitatively 
alter the drift trajectory, suggesting that weight drift is 
driven primarily by the activation bias rather than by gradient 
magnitude.

\textit{Sparsity.}
Activation sparsity increases organically throughout training 
without any explicit regularization. In the plain MLP, ReLU and 
NoisyReLU reach the highest sparse-neuron fractions 
(${\sim}0.6$), consistent with ReLU's hard thresholding at zero. As weights become more negative, pre-activations 
shift downward, causing more units to fall below the activation 
threshold. Adding a skip connection substantially disrupts this 
mechanism, with ReLU and NoisyReLU sparsity dropping from 
${\sim}0.6$ to ${\sim}0.4$, at the cost of denser, 
higher-magnitude gradient flow.

\begin{figure}[ht]
    \centering
    \includegraphics[width=1\linewidth]{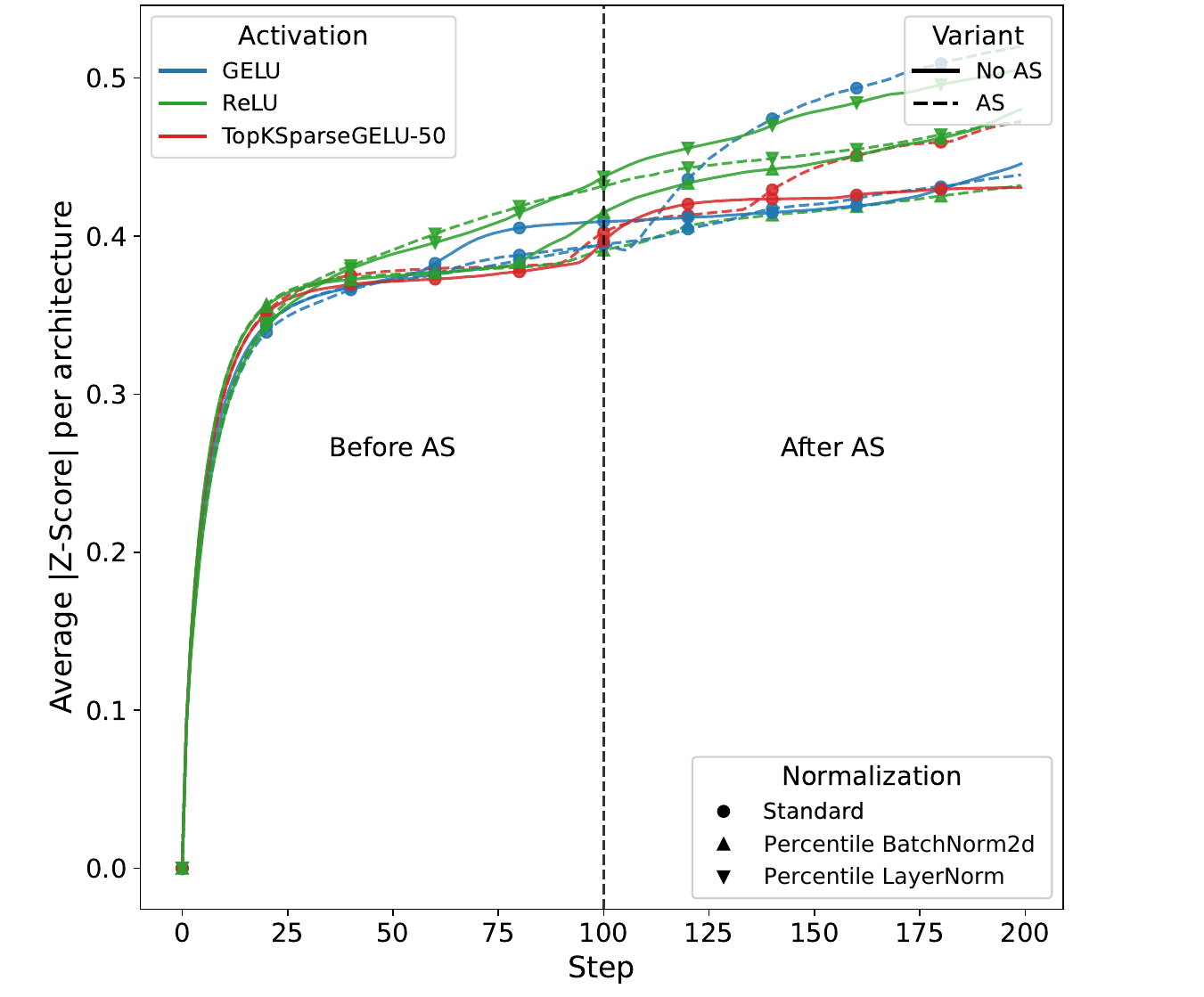}
    \caption{
        \textbf{Weight drift before and after Accumulation 
        Stop.} Average weight drift (mean $|$Z-score$|$) 
        for GELU, ReLU, and TopKSparseGELU-50 under 
        Standard, PercentileBN, and PercentileLN 
        normalization strategies. The dashed vertical line 
        marks the AS boundary at step 100. Solid lines: 
        without AS. Dashed lines: with frozen EMA 
        statistics. Trajectories remain continuous and 
        stable after the boundary across all configurations.
    }
    \label{fig:weight_drift_as}
\end{figure}

\subsubsection{Weight drift when normalization statistics are frozen}
\label{app:as_drift}
Figure~\ref{fig:weight_drift_as} shows that freezing normalization statistics with Accumulation Stop (\S\ref{sec:efficiency}) does not disturb weight drift.
\begin{figure*}[t]
    \centering
    \includegraphics[width=0.95\linewidth]{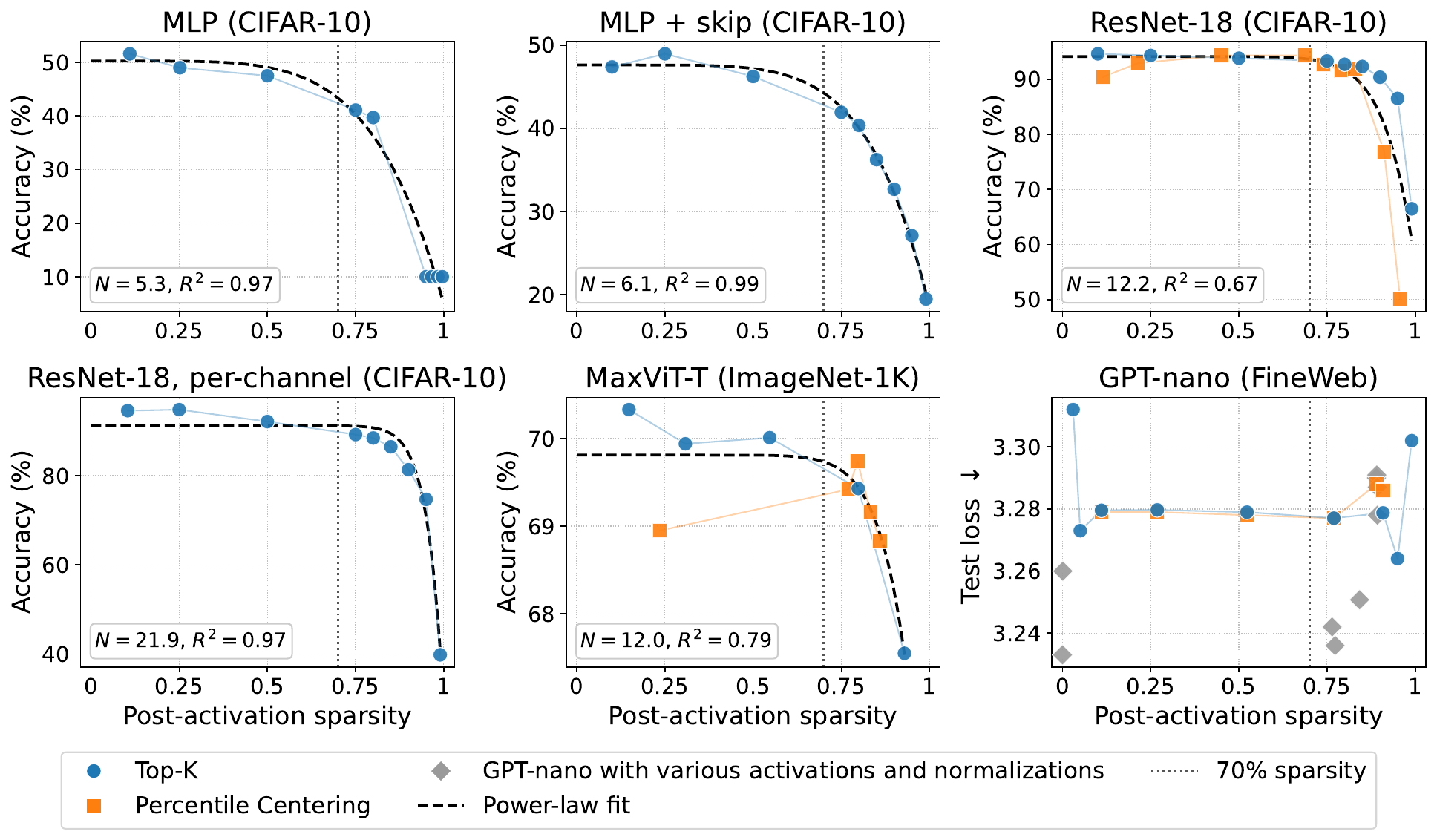}
    \caption{\textbf{Sparsity--accuracy tradeoff per architecture.} The same 79 configurations
    as in Figure~\ref{fig:sparsity_predictor}, shown as raw accuracy (test loss for GPT-nano)
    rather than as a pooled normalized score. In the classification models, Top-K is applied to
    GELU and Percentile Centering to ReLU. Dashed curves show a per-panel fit $A - B\,s^{N}$.
    No fit is shown for GPT-nano, whose loss stays flat up to $s \approx 0.9$. Gray diamonds
    are GPT-nano runs with other activation functions or normalization layers at their natural
    sparsity. The dotted line marks 70\% sparsity.}
    \label{fig:sparsity_per_arch}
\end{figure*}

\subsection{Controllable sparsity}
\label{sec:controlled_sparsity}

Tables~\ref{tab:pruning_results_resized} and~\ref{tab:pruning_results_extended}
report accuracy (or test loss for GPT), negative pre-activation
fractions, and post-activation sparsity across architectures and sparsity
levels under both Percentile Centering and Top-K pruning.

\paragraph{Percentile Centering and Top-K are consistent.}
Both methods produce closely matching sparsity levels and negative value
fractions at equivalent percentile targets, confirming their practical
equivalence and supporting the conclusion of
\S\ref{sec:sparsity_predictor} that sparsity level rather than
mechanism is the dominant predictor of performance.

\paragraph{Architecture determines sparsity tolerance.}
Figure~\ref{fig:sparsity_per_arch} shows the sparsity--accuracy tradeoff separately for each architecture.
ResNet retains 94.3\% accuracy at 75\% per-activation sparsity and
degrades only to 90.4\% at 90\%, while structured (per-channel)
sparsity incurs a heavier penalty, dropping to 74.7\% at 95\% due
to the loss of entire feature channels. ViT accuracy varies by less
than 3~pp ($67.6$--$70.3\%$) across all evaluated levels under both methods, consistent
with attention-based aggregation being inherently tolerant to sparse
MLP activations. GPT test loss remains nearly flat from 11\%
to 77\% sparsity (increasing by less than 0.002 nats) and degrades
only marginally to 3.286 at 90\%, suggesting the accuracy cliff for
transformer language models lies well above the range evaluated here.
By contrast, the plain MLP collapses to near-chance performance
(${\sim}10\%$) beyond 80\% sparsity, while adding skip connections
allows it to retain 32.7\% accuracy even at 90\%, confirming that
residual paths provide a bypass mechanism under aggressive
sparsification.

\paragraph{Negative value fractions confirm weight drift.}
Pre-activation negative fractions are substantial across all models
even before any sparsity is explicitly induced, corroborating the
drift findings of \S\ref{sec:weight_drift}.
Per-channel and per-activation ResNet results share similar
negative fractions at equivalent sparsity levels.
 ViT exhibits
particularly high negative fractions (${\sim}0.77$--$0.86$) at
low sparsity levels.

\paragraph{Sparsification in generative models.}
Figure~\ref{fig:dit_generations} extends the sparsity analysis to
image generation with DiT-S/2~\citep{peebles2023dit} on
ImageNet. All configurations converge to a plausible overall
composition within 50K steps and improve progressively to 300K
steps. Differences emerge primarily in high-frequency detail.
The baseline ReLU configuration produces noticeably blurred
backgrounds and lacks fine-grained structure, while GELU with
Top-K sparsity consistently preserves sharp structural details
and textural realism. This suggests that Top-K masking of a
smooth non-monotonic activation such as GELU may better protect
the high-frequency spatial gradients needed for clear visual
boundaries, offering a practical recommendation for
sparsity-aware generative model design.

\begin{table*}[ht]

\centering
\caption{Results when sparsity is induced with percentile shift or TOP-K. Negative values and sparsity   for GPT are computed only after down-projection and before up-projection layers correspondingly. Results are presented as: (\textbf{Accuracy or \rev{test} loss for GPT}) / \ (\textbf{Negative Values before activation function}) / (\textbf{Sparsity after activation function}).}
\resizebox{\textwidth}{!}{%
    \begin{tabular}{lccccc}
    \toprule
    \textbf{Model} & \textbf{P10} & \textbf{P25} & \textbf{P50} & \textbf{
    P75} & \textbf{P90} \\
    \cmidrule(lr){2-6}
    \multicolumn{6}{c}{\textit{Percentile Centering - \textbf{ReLU} is used to induce sparsity}} \\
    \midrule
    MLP & 45\% / 0.123 / 0.121 & 48.3\% / 0.273 / 0.272 & 47.0\% / 0.502 / 0.500 & 43.3\% / 0.747 / 0.750 & 24.6\% / 0.889 / 0.880 \\
    MLP + SKIP & 41.75 / 0.137 / 0.132 & 44.64 / 0.269 / 0.271 & 45.11 / 0.503 / 0.504 & 47.58 / 0.737 / 0.737 & 41.12 / 0.876 / 0.873 \\
    ResNet  & 90.4\% / 0.115 / 0.115 & 92.9\% / 0.215 / 0.214 & 94.3\% / 0.451 / 0.450 & 94.3\% / 0.687 / 0.687 & 91.8\% / 0.849 / 0.829 \\
    MaxViT& 68.95\% / 0.761 / 0.236 & 69.42\% / 0.770 / 0.770 & 69.74\% / 0.796 / 0.796 & 69.16\% / 0.833 / 0.833 & 68.83\% / 0.859 / 0.859 \\
    \midrule
    \multicolumn{6}{c}{} \\ 
    \midrule
    \multicolumn{6}{c}{\textit{TOP-K Pruning - \textbf{GELU} is used for all models}} \\
    \cmidrule(lr){2-6}
    \textbf{Model} & \textbf{K90} & \textbf{K75} & \textbf{K50} & \textbf{K25} & \textbf{K10} \\
    \midrule
    MLP  & 51.6\% / 0.610 / 0.110 & 49.0\% / 0.626 / 0.252 & 47.5\% / 0.661 / 0.500 & 41.1\% / 0.595 / 0.750 & 10.0\% / 0.500 / 0.966 \\
    MLP  + SKIP & 47.37 / 0.480 / 0.100 & 48.95 / 0.485 / 0.250 & 46.23 / 0.501 / 0.500 & 41.92 / 0.545 / 0.750 & 32.68 / 0.603 / 0.900 \\
    ResNet  & 94.6\% / 0.469 / 0.100 & 94.3\% / 0.471 / 0.250 & 93.8\% / 0.447 / 0.500 & 93.3\% / 0.436 / 0.750 & 90.35\% / 0.370 / 0.900 \\
    MaxViT& 70.33\% / 0.788 / 0.148 & 69.94\% / 0.771 / 0.308 & 70.01\% / 0.771 / 0.547 & 69.43\% / 0.772 / 0.798 & 67.55\% / 0.786 / 0.929 \\
     GPT  & 3.279 / 0.1106 / 0.1106 & 3.279 / 0.2689 / 0.2689 & 3.278 / 0.5231 / 0.5231 & 3.277 / 0.7692 / 0.7692 & 3.286 / 0.9087 / 0.9087 \\
    \bottomrule
    \end{tabular}%
} 

\label{tab:pruning_results_resized}
\end{table*}

\begin{table*}[h!]
\label{tab:sparse_2}
\centering
\caption{Extended results when sparsity is induced with Percentile 
Centering or TOP-K across additional sparsity levels. Results are 
presented as: (\textbf{Accuracy}) / (\textbf{Negative Values before 
activation function}) / (\textbf{Sparsity after activation function}).}
\resizebox{\textwidth}{!}{%
    \begin{tabular}{lcccccccc}
    \toprule
    \textbf{Model} & \textbf{Q10} & \textbf{Q25} & \textbf{Q50} & 
    \textbf{Q75} & \textbf{Q80} & \textbf{Q85} & \textbf{Q90} & 
    \textbf{Q95} \\
    \cmidrule(lr){2-9}
    \multicolumn{9}{c}{\textit{Percentile Centering with \textbf{ReLU} 
    is used to induce sparsity}} \\
    \midrule
    ResNet 
        & 90.4\% / 0.115 / 0.115 
        & 92.9\% / 0.215 / 0.214 
        & 94.3\% / 0.451 / 0.450 
        & 94.3\% / 0.687 / 0.687 
        & 92.6\% / 0.740 / 0.740 
        & 91.6\% / 0.791 / 0.791 
        & 91.8\% / 0.849 / 0.829 
        & 76.8\% / 0.912 / 0.912 \\
    \midrule
    \multicolumn{9}{c}{} \\
    \midrule
    \multicolumn{9}{c}{\textit{TOP-K Pruning -- \textbf{GELU} is 
    used for all models and BN is applied in ResNet-18}} \\
    \cmidrule(lr){2-9}
    \textbf{Model} & \textbf{K90} & \textbf{K75} & \textbf{K50} & 
    \textbf{K25} & \textbf{K20} & \textbf{K15} & \textbf{K10} & 
    \textbf{K5} \\
    \midrule
    ResNet 
        & 94.6\% / 0.469 / 0.100 
        & 94.3\% / 0.471 / 0.250 
        & 93.8\% / 0.447 / 0.500 
        & 93.3\% / 0.436 / 0.750 
        & 92.7\% / 0.406 / 0.800 
        & 92.3\% / 0.388 / 0.850 
        & 90.4\% / 0.370 / 0.900 
        & 86.5\% / 0.367 / 0.950 \\
    ResNet Struct. 
        & 94.6\% / 0.600 / 0.104 
        & 94.8\% / 0.578 / 0.250 
        & 92.1\% / 0.543 / 0.500 
        & 89.2\% / 0.511 / 0.750 
        & 88.4\% / 0.506 / 0.800 
        & 86.5\% / 0.519 / 0.850 
        & 81.3\% / 0.514 / 0.900 
        & 74.7\% / 0.512 / 0.950 \\
    MLP 
        & 51.6\% / 0.610 / 0.110 
        & 49.0\% / 0.626 / 0.252 
        & 47.5\% / 0.661 / 0.500 
        & 41.1\% / 0.595 / 0.750 
        & 39.7\% / 0.490 / 0.800 
        & 10.0\% / 0.550 / 0.950 
        & 10.0\% / 0.500 / 0.966 
        & 10.0\% / 0.501 / 0.983 \\
    MLP  + SKIP 
        & 47.4\% / 0.480 / 0.100 
        & 49.0\% / 0.250 / 0.250 
        & 46.2\% / 0.501 / 0.500 
        & 41.9\% / 0.545 / 0.750 
        & 40.4\% / 0.563 / 0.800 
        & 36.2\% / 0.587 / 0.850 
        & 32.7\% / 0.603 / 0.900 
        & 27.1\% / 0.576 / 0.950 \\
    \bottomrule
    \end{tabular}%
}
\label{tab:pruning_results_extended}
\end{table*}

\subsection{Spikes and activation functions}
\label{app:spikes_activations}

\subsubsection{Spikes in GPT-nano}
\label{app:gpt_extended}
Tables~\ref{tab:std}--\ref{tab:maxmin} collectively support the hypothesis that
activation spikes originate in the MLP block rather than in attention layers.
Weight standard deviations remain stable across all layer types, including in
spike layers L1--L4, ruling out
weight instability as a cause. The down-projection consistently receives
large, one-sided inputs (max $\approx 39$--$42$, min $\approx 0$) in both the
full model and spike layers, a direct consequence of the activation function
zeroing negatives and amplifying large positives. By contrast, attention
layers exhibit symmetric, moderate input ranges, arguing against any causal
role of the attention mechanism or Softmax in spike formation.

\begin{figure}[h]
    \centering
    \includegraphics[width=\linewidth]{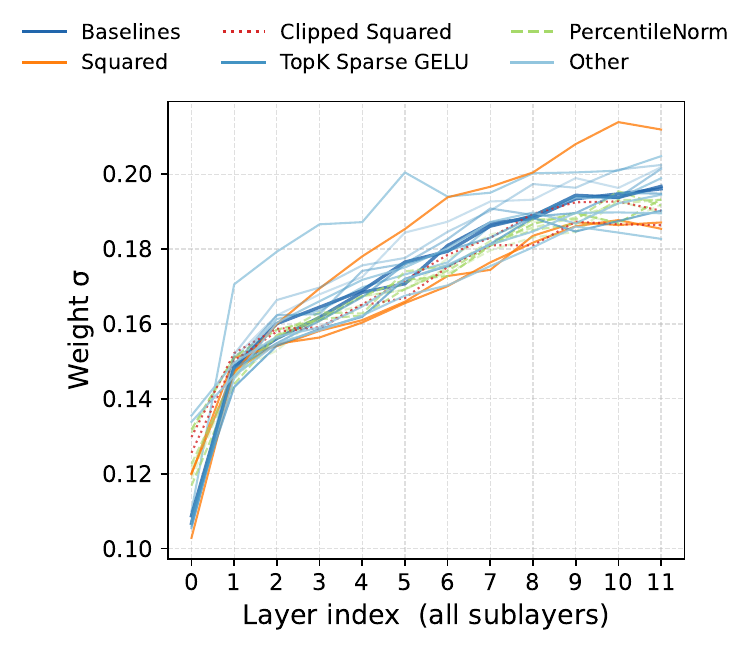}
    \caption{Weight standard deviation averaged across sublayers per layer index, aggregated over 23 runs for different activation functions, normalization and sparsification strategies. All runs show a consistent monotonic increase. Squared activation (orange) exhibiting slightly elevated values in later layers.}
    \label{fig:layer_wise_std}
\end{figure}
\paragraph{Origin of spikes.}
Down-projection weights in spike layers are not larger than elsewhere in the network (max $|W| = 2.299$ vs.\ $2.858$ overall, Table~\ref{tab:maxmin}), and up-projection inputs are symmetric and moderate (${\approx}\pm 7$). The asymmetry therefore appears exactly where the activation function acts, which indicates that spikes arise because a few token representations align with particular rows of the up-projection. Our measurements do not explain why this alignment concentrates on specific tokens and layers. \citet{sun2026spike} attribute spikes in SwiGLU models to quadratic amplification from the gate--up product, but our non-gated MLP shows that a gate is not required. \citet{barbero2025first} argue that the first or ``special'' tokens take on a functional role that limits over-mixing, which is consistent with the concentration of spikes in early layers.

\paragraph{Seed variability of squared activations.}
Table~\ref{tab:relu2_seeds} reports two additional runs of ReLU$^2$ 
alongside the first run from Table~\ref{tab:activations_range}. 
Spikes are present in all three runs, 
but their magnitude varies considerably. The Layer~2 range drops from $1055.0$ in the 
first run to $248.7$ and $178.3$ in the additional runs. 
Even the smaller Layer~2 values remain 4--6$\times$ above ReLU and GELU. Test loss also varies, 
from $3.251$ to $3.233$. In Run~2 it coincides with GELU$^2$ ($3.233$) 
despite a pronounced spike. ReLU$^2_{\text{clip50}}$, in contrast,
 reaches $3.236$ with negligible variance across three runs.
  We conclude that spike amplification by squared activations is consistent across
   seeds but not always damaging. Clipping removes it and makes training outcomes more
    predictable without degrading loss. 
    The advantage of squared over non-squared activations (Figure~\ref{fig:gpt_results}) 
    holds across all runs.
\begin{figure*}[h]
    \centering
    \includegraphics[width=0.9\linewidth]{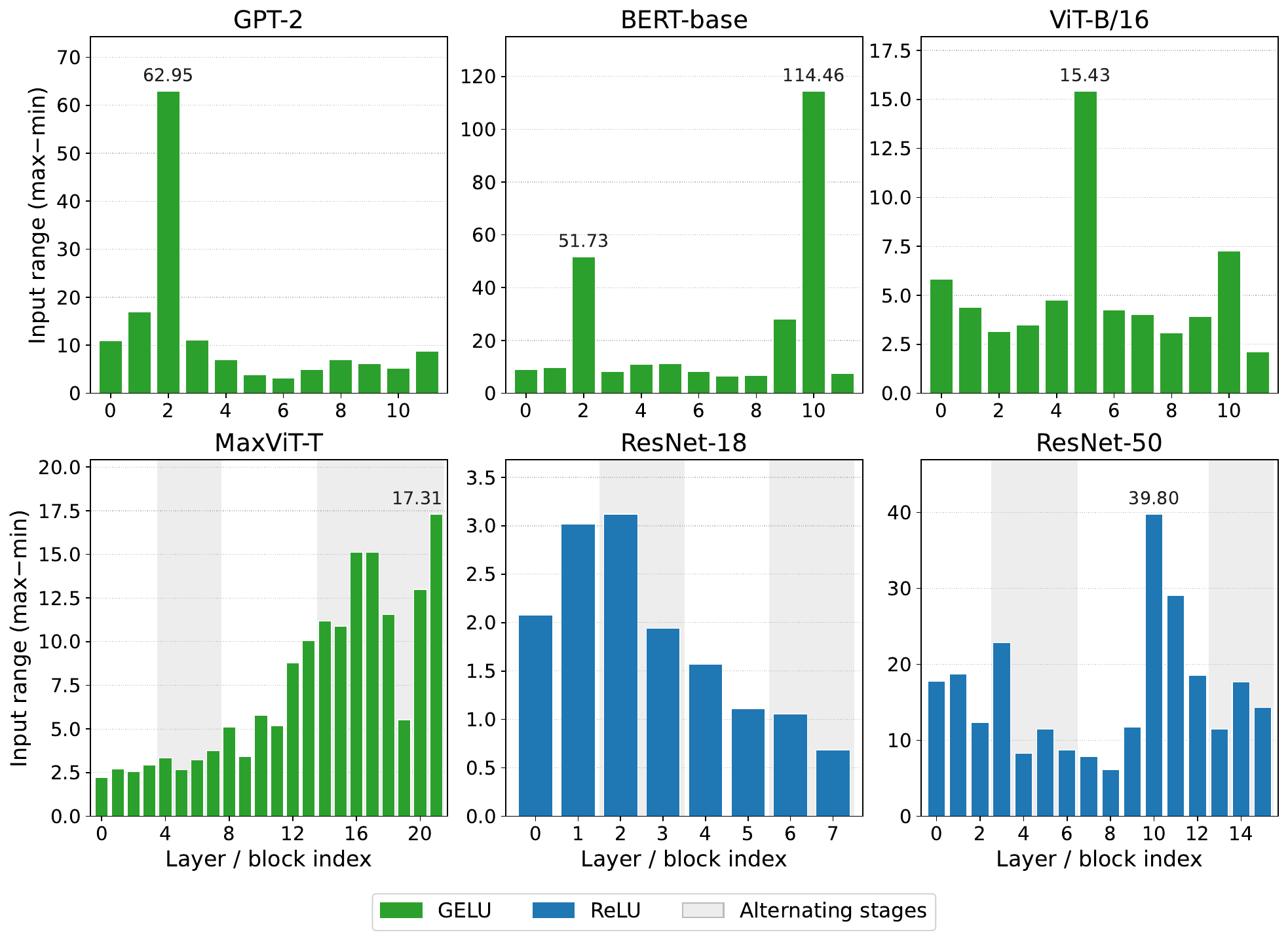}
    \caption{Down-projection input range (max$-$min) per layer for six pretrained models. Spike values are labeled.
    Shaded bands mark alternating stages. MaxViT-T has two MLP sublayers per block,
    after window and after grid attention. ResNet indices count residual blocks
    (2 per stage in ResNet-18, and 3, 4, 6 and 3 in ResNet-50).}
    \label{fig:spikes_pretrained}
\end{figure*}

\begin{table}[ht]
\centering
\caption{Loss and per-layer down-projection input range (max$-$min) for three ReLU$^2$ seeds on GPT-nano, compared to single runs of ReLU and GELU. ReLU$^2_{\text{clip50}}$ reaches a test loss of $3.236$ with negligible variance across three seeds.}
\label{tab:relu2_seeds}
\resizebox{\linewidth}{!}{%
\begin{tabular}{lccccc}
\toprule
 & \textbf{ReLU} & \textbf{GELU} & \textbf{ReLU$^2$ Run 1} & \textbf{ReLU$^2$ Run 2} & \textbf{ReLU$^2$ Run 3} \\
\midrule
Loss $\downarrow$ & 3.290 & 3.260 & 3.251 & 3.233 & 3.236 \\
\midrule
\quad Layer 1 & 16.2 & 12.4 & 108.5  & 33.0  & 29.5  \\
\quad Layer 2 & 42.2 & 42.7 & 1055.0 & 248.7 & 178.3 \\
\quad Layer 3 & 67.5 & 28.5 & 260.9  & 96.3  & 164.5 \\
\quad Layer 4 & 17.8 & 10.2 & 49.8   & 131.3 & 59.4  \\
\quad Layer 5 & 13.1 & 16.0 & 65.3   & 24.5  & 32.0  \\
\bottomrule
\end{tabular}%
}
\end{table}
\begin{table*}[ht]
\centering
\caption{Weight standard deviations averaged over 23 runs, for all layers and spike layers (L1--L4).
All layer types show lower and tighter STDs in spike layers, indicating weight stability
precisely where activation spikes occur.}
\label{tab:std}
\resizebox{0.7\linewidth}{!}{%
\begin{tabular}{lcccccc}
\toprule
 & \multicolumn{3}{c}{\textbf{All Layers}} & \multicolumn{3}{c}{\textbf{Layers 1--4}} \\
\cmidrule(lr){2-4} \cmidrule(lr){5-7}
\textbf{Layer Type} & \textbf{Avg} & \textbf{Min} & \textbf{Max} & \textbf{Avg} & \textbf{Min} & \textbf{Max} \\
\midrule
\texttt{attn.out\_proj}       & 0.1919 & 0.0208 & 0.3073 & 0.1668 & 0.0208 & 0.2349 \\
\texttt{mlp.up\_projection}   & 0.1616 & 0.0208 & 0.1819 & 0.1662 & 0.0208 & 0.1811 \\
\texttt{mlp.down\_projection} & 0.1583 & 0.0104 & 0.2017 & 0.1468 & 0.0104 & 0.1727 \\
\texttt{attn.qkv\_projection} & 0.1481 & 0.0208 & 0.2284 & 0.1353 & 0.0208 & 0.1797 \\
\bottomrule
\end{tabular}%
}
\end{table*}

\begin{table*}[ht]
\centering
\caption{Average maximum and minimum weight and input values over 23 runs, for all layers
and spike layers (L1--L4). The down-projection consistently receives large, one-sided inputs
(max $\approx$39--42, min $\approx$0), while attention layers exhibit symmetric and moderate ranges.}
\label{tab:maxmin}
\resizebox{\linewidth}{!}{%
\begin{tabular}{lcccccccc}
\toprule
 & \multicolumn{4}{c}{\textbf{All Layers}} & \multicolumn{4}{c}{\textbf{Layers 1--4}} \\
\cmidrule(lr){2-5} \cmidrule(lr){6-9}
\textbf{Layer Type} & \textbf{Max W} & \textbf{Min W} & \textbf{Max In} & \textbf{Min In}
                    & \textbf{Max W} & \textbf{Min W} & \textbf{Max In} & \textbf{Min In} \\
\midrule
\texttt{attn.out\_proj}       & $1.139$ & $-1.143$ & $19.33$ & $-19.23$
                              & $0.934$ & $-0.961$ & $15.60$ & $-15.67$ \\
\texttt{attn.qkv\_projection} & $1.189$ & $-1.185$ & $9.05$  & $-7.87$
                              & $0.930$ & $-0.925$ & $7.59$  & $-6.47$ \\
\texttt{mlp.down\_projection} & $\mathbf{2.858}$ & $\mathbf{-2.781}$ & $\mathbf{39.01}$ & ${\approx}0$
                              & $\mathbf{2.299}$ & $\mathbf{-1.957}$ & $\mathbf{42.00}$ & ${\approx}0$ \\
\texttt{mlp.up\_projection}   & $1.245$ & $-1.264$ & $7.47$  & $-6.65$
                              & $1.136$ & $-1.136$ & $5.30$  & $-4.86$ \\
\bottomrule
\end{tabular}%
}
\end{table*}

\subsubsection{Spikes in other pretrained language and vision models}
\label{app:pretrained_spikes}
We further extend the spike analysis to six 
 publicly available pretrained models to evaluate whether spikes are a general phenomenon 
 or specific to some models. Models span different architectures 
 (transformers, a hybrid convolution--attention model, and CNNs), two modalities (text and images), 
 and training objectives. For each model we measure the down-projection input range (max$-$min) per layer 
 (Figure~\ref{fig:spikes_pretrained}).

\paragraph{Findings.}
Anomalous spikes appear only in the language models, GPT-2 and BERT-base, where the largest value
is 9 to 12 times the median value over layers. The vision models, ViT-B/16, MaxViT-T, ResNet-18 and ResNet-50,
show some variation across layers, but no layer exceeds 4 times the median. Overall, these results suggest that early-layer
  spikes are a property of language models rather than of the transformer MLP block itself.
   This supports the hypothesis that they originate from tokens with a ``functional''
    role, such as the first or special tokens~\citep{barbero2025first}.

\subsubsection{Squared activation functions in ViT}
\label{app:squared_vision}
Here we test squared activation functions in  the vision transformer MaxViT.
Table~\ref{tab:maxvit_results} reports the results. Unlike in GPT-nano, squared activations do not
outperform their non-squared counterparts on MaxViT and lose 7 to 10~pp of accuracy. Clipping at 50
recovers most of this gap, although the clipped variants remain slightly below plain ReLU and GELU.
\begin{figure}[h]
    \centering
    \includegraphics[width=\linewidth]{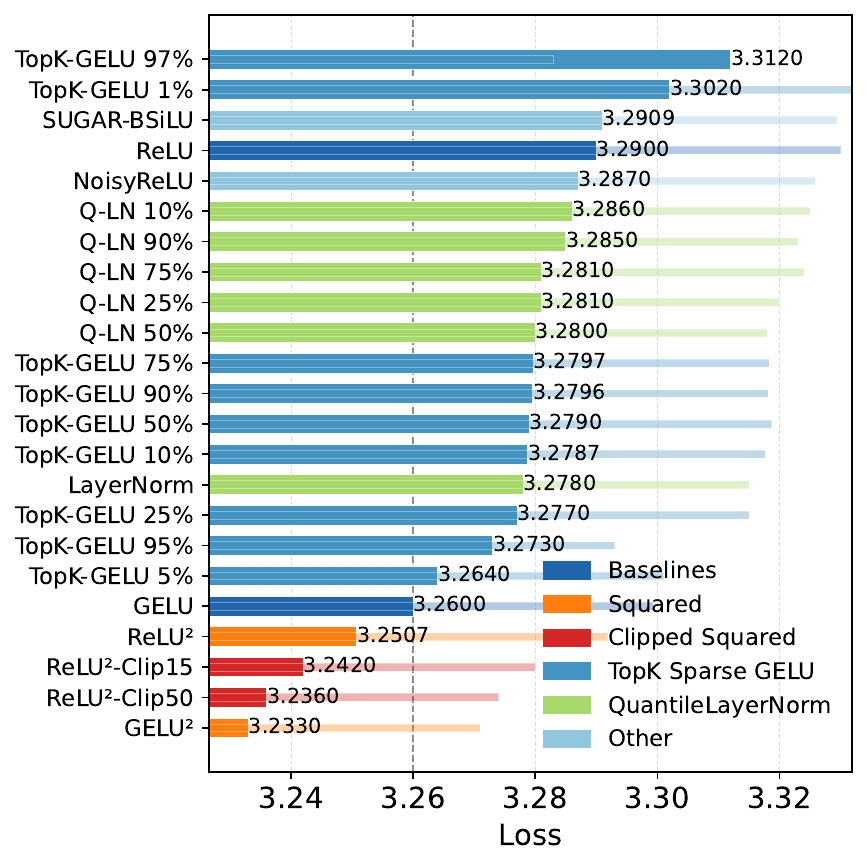}
    \caption{GPT-nano performance with different activation functions, sparsification and normalization approaches, transparent bars reflect train loss and solid lines reflect \rev{test} loss.}
    \label{fig:gpt_results}
\end{figure}

\begin{table*}[h]
\centering
\caption{Aggregated results on MaxViT (accuracy $\uparrow$ / sparsity).}
\label{tab:maxvit_results}
\resizebox{0.95\linewidth}{!}{%
\begin{tabular}{lcccccccc}
\toprule
 & ReLU & GELU & NoisyReLU & SUGARBSiLU & ReLU$^2$ & ReLU$^2_{\text{clip50}}$ & GELU$^2$ & GELU$^2_{\text{clip50}}$ \\
\midrule
Acc / Sparse &
69.57 / 0.798 &
70.30 / 0.010 &
68.46 / 0.172 &
51.49 / 0.829 &
62.38 / 0.756 &
69.25 / 0.669 &
60.65 / 0.017 &
69.63 / 0.003 \\
\bottomrule
\end{tabular}
}
\end{table*}

\subsection{Accumulation Stop results}
\label{app:as_results}
Tables~\ref{tab:throughput} and~\ref{tab:quality} give the full throughput and quality results summarized in Figure~\ref{fig:accumulation_stop} (\S\ref{sec:efficiency}).

\begin{table*}[t!]
\centering
\begin{minipage}[t]{0.49\linewidth}
\centering
\caption{Training throughput (batches/sec, $\uparrow$) for 
DiT-S/2 (H200) and MaxViT-Tiny (A100) before and after 
Accumulation Stop ($T_{\mathrm{warm}}=50{,}000$ steps, 
\rev{non-optimized} implementation).}
\label{tab:throughput}
\resizebox{\linewidth}{!}{%
\begin{tabular}{llcc}
\toprule
\textbf{Activation} & \textbf{Norm} & 
\textbf{Accum.} & \textbf{Fixed} \\
\midrule
\multicolumn{4}{l}{\textit{DiT-S/2 --- baseline: 3.42 
(GELU + PyTorch LN)}} \\
\midrule
Percentile GELU & LN    & 2.47 & 3.43 \\
GELU            & Percentile LN & 2.47 & 3.41 \\
ReLU            & Percentile LN & 2.58 & 3.53 \\
\midrule
\multicolumn{4}{l}{\textit{MaxViT --- baseline: 1.59 
(GELU + PyTorch BN)}} \\
\midrule
GELU            & BN         & 1.55 & \textbf{2.09} \\
Percentile GELU & BN    & 1.34 & \textbf{2.13} \\
ReLU            & Percentile BN & 1.14 & \textbf{1.73} \\
\bottomrule
\end{tabular}%
}
\end{minipage}\hfill
\begin{minipage}[t]{0.49\linewidth}
\centering
\caption{Model quality when Accumulation Stop is applied. DiT-S/2 
reports FID$\downarrow$ / IS$\uparrow$ (\textbf{baseline: 49.40 / 
29.85}), MaxViT reports top-1 accuracy$\uparrow$ 
(baseline: 70.30\%).}
\label{tab:quality}
\resizebox{\linewidth}{!}{%
\begin{tabular}{llc}
\toprule
\textbf{Activation} & \textbf{Norm} & \textbf{Score} \\
\midrule
\multicolumn{3}{l}{\textit{DiT-S/2}} \\
\midrule
Percentile GELU (50\%) & LN     & 52.33 / 28.87 \\
GELU                   & Percentile Shift (50\%) & \textbf{48.21} / \textbf{31.41} \\
ReLU                   & Percentile Shift (50\%) & 51.44 / 29.00 \\
\midrule
\multicolumn{3}{l}{\textit{MaxViT}} \\
\midrule
GELU                   & BN                 & 69.75\% \\
Percentile GELU (50\%) & BN          & 69.96\% \\
ReLU                   & Percentile BN (50\%) & 69.74\% \\
\bottomrule
\end{tabular}%
}
\end{minipage}
\end{table*}

\section{Extended Limitations and Discussion}
\label{app:discussion}

This appendix expands on the limitations and open questions summarized in
\S\ref{sec:limitations}. We address each of the six points in turn,
following the numbering used in the main text.

\subsection{Limitations}

\paragraph{(1) Theoretical assumptions hold strictly only at initialization.}
Our theorems assume zero-mean i.i.d.\ weights, independence between $V^{(l)}_{\mathrm{eff}}$ and
$p^{(l)}$, and, for cross-entropy, a softmax linearization around $f = 0$.
Appendix~\ref{app:beyond_init} argues that the drift persists after initialization, but this
argument is informal and holds only while the weights stay close to their initial values.

\paragraph{(2) Formal proof covers ReLU only.}
The argument in Theorem~\ref{th:veff_properties} relies on the
survival-conditioning property of ReLU gates: each $D_l$ is a binary diagonal
matrix that selects active neurons, and conditioning on survival induces a
positive correlation along the input direction. Smooth activations such as GELU
and SiLU do not admit a clean binary gating decomposition, so extending the
proof requires a continuous analogue of survival conditioning, which is beyond
the scope of this paper. Empirically, the same drift pattern is observed across
GELU, SiLU, NoisyReLU, and SUGARBSiLU, except MaxViT-Tiny under GELU and SiLU.


\paragraph{(3) Empirical scope and clipping thresholds.}
Classification experiments are conducted on CIFAR-10 and ImageNet-1K, and
language modeling uses a single dataset (FineWeb) with a small autoregressive
model (GPT-nano, 124M parameters). Whether the $\sim$70\% sparsity cliff and
the benefit of clipped squared activations transfer unchanged to
frontier-scale LLMs (1B+ parameters, multi-trillion-token training) remains
open. The profiling of pretrained GPT-2 and BERT (Appendix~\ref{app:pretrained_spikes}) and the sparsity levels reported for 7B-scale models in prior work are consistent with our findings, but direct training experiments at larger scale remain future work. The cliff itself is fitted from $N{=}79$ configurations and is consistent
across the architectures we tested, but a denser sweep at intermediate scales
would help characterize how the cliff position depends on capacity, depth, and
training duration. Clipping thresholds (15 and 50) for ReLU\textsuperscript{2}
were selected empirically on GPT-nano based on the spike magnitudes observed. 
In practice, the optimal threshold might be
architecture-dependent. Adaptive or layer-wise clipping schedules, possibly
tied to running activation statistics, are a natural extension we did not
explore.

\paragraph{(4) Throughput gains are conservative lower bounds.}
Accumulation Stop is evaluated on architectures with fixed-size inputs
(DiT-S/2, MaxViT), where the warm-up statistics provide a stable and
representative summary of the full training distribution. Its applicability to
autoregressive models is more limited. Reported
throughput in Table~\ref{tab:throughput} compares a naive PyTorch
implementation of percentile centering against PyTorch's optimized
LayerNorm/BatchNorm kernels. Because our implementation is unoptimized while
the baseline is fused, the recovery to near-baseline throughput already
represents a conservative lower bound. A fused CUDA kernel for percentile
centering, analogous to existing fused LayerNorm kernels, would likely yield
further improvement and could make \textbf{Accumulation Stop a net throughput win} even
during the warm-up phase.

\paragraph{(5) Transformer robustness to sparsification.}
ViT and GPT-nano tolerate sparsity levels far beyond the cliff observed in
MLPs, with GPT-nano \rev{test} loss remaining nearly flat up to $s \approx
0.91$. Comparable levels of sparsity in text transformers have been reported in prior work that exploits rather than explains the phenomenon~\citep{song2024prosparse, liu2024training, song2024turbo, lee2024cats}. Robustness to sparsity is not limited to transformers. MaxViT accuracy drops only from 70\% to 67\% as activation sparsity increases from 10\% to 90\%, and ResNet drops from 94\% to 90\% over the same range.
Skip connections improve sparsity tolerance in MLPs (from chance accuracy to 27.1\% at
95\% sparsity), but this gap alone does not
account for the much greater robustness of attention-based architectures.
Several factors likely contribute, including attention as a content-based
aggregation mechanism that can route around sparsified MLP outputs,
overparameterization of the feed-forward blocks relative to the attention
pathway, and the residual structure that makes the MLP block an additive
correction rather than a serial bottleneck. The residual path preserves the signal, so the next layer effectively receives $\text{sparse}(\mathbf{x}) + \text{residual}$, and updating only ${\sim}10\%$ of units in each block appears to be sufficient. This hypothesis could be tested by increasing the number of residual blocks at a fixed sparsity level. A complementary direction is to track geometric properties of intermediate representations, such as effective rank or intrinsic dimensionality per layer, across sparsity levels. A sharp drop in intrinsic dimensionality coinciding with the accuracy cliff would indicate that sparsity directly limits representational capacity. Note that a representation of dimension $N$ can in principle retain rank $N$ even under very high sparsity, for example with a single active unit per position across different inputs. If dimensionality is instead preserved despite the accuracy drop, the cause is more likely poor gradient dynamics, such as too few active units to sustain a useful learning signal.
Disentangling these factors
requires controlled ablations beyond what this paper covers.

\paragraph{(6) DiT-S/2 improvements with Percentile LayerNorm.}
On DiT-S/2, GELU with Percentile LayerNorm at the 50th percentile achieves
FID 48.21 and IS 31.41, both improving over the centered LayerNorm baseline
(FID 49.40, IS 29.85). We do not characterize this effect mechanistically. There is one
possible explanation. Gradients flow differently through a quantile and through a mean, since for a quantile they flow only through a single value.

\subsection{Additional discussion}
\label{app:additional_discussion}

This subsection collects further discussion points that arose during the review process.

\paragraph{Drift and mean-centering normalization.}
Mean-centering normalization (BatchNorm, LayerNorm) suppresses weight drift by 
removing the positive pre-activation bias that drives it. 
ResNet with BatchNorm is the one configuration in Table~\ref{tab:activations} 
without substantial negative pre-activations. Our findings therefore apply 
 to architectures without centering, that includes most recent LLMs,
  which use RMSNorm (\S\ref{app:normalizations}).

\paragraph{Does the sparsity cliff shift with scale?}
Activation sparsity in large models is consistent with our GPT-nano results.
 LLMs at the 7B scale can be pushed to ${\sim}90\%$ sparsity: \citet{song2024prosparse} report 89.3\% 
 on LLaMA2-7B, and \citet{song2024turbo} report ${\sim}90\%$ on Mistral-7B. More broadly, 
 prior work reports sparsity of 50--90\% in models such as Llama-2, Llama-3, Mistral 7B--70B,
  and MiniCPM-1B~\citep{liu2023deja, song2024prosparse, liu2024training, song2024turbo, lee2024cats}.
   Mixture-of-Experts models can be read as structured activation sparsity, 
   and the numbers coincide again. Mixtral~\citep{jiang2024mixtral}
    activates 2 of 8 experts per layer, i.e., 75\% sparsity, and recent MoEs
     route more aggressively still. We therefore do not observe a clear progression 
     of sparsity tolerance with model size. We expect the position of the cliff to depend 
     on the task and objective, the architecture (e.g., the presence of skip connections),
      and the amount of training data at least as much as on capacity. For example,
       it is difficult to imagine that diffusion-based image generation models would tolerate
        similarly high sparsity.

\paragraph{Why the first layer does not drift.}
The expected update to $W^{(l)}_{ik}$ factorizes as $\mathbb{E}[(\partial\ell/\partial p_i^{(l)}) \cdot x_k^{(l-1)}]$ (Remark in Appendix~\ref{sec:ce_proof}). 
The first factor is non-negative in expectation by Theorems~\ref{th:pos_p_gradient} 
and~\ref{th:pos_p_gradient_ce}, so the sign of the expected update is determined by 
the sign of the layer input. For the first layer, the input is the data, 
which is often mean-centered. The product then has no consistent sign and no drift accumulates. 
For every subsequent layer, the input is a post-activation and hence non-negative (ReLU) 
or positively biased (GELU, SiLU). The expected update is then non-positive and drift accumulates. 
In a transformer, each MLP block has a single hidden layer, but its down-projection receives
 a post-activation input and therefore satisfies exactly this condition, and the block 
 is repeated in every layer.

\paragraph{Would a drift-free activation perform better?}
Symmetric, zero-centered activations such as tanh do not induce a negative drift, 
yet they perform considerably worse in practice. More importantly, 
our sparsity analysis suggests that negatively shifted weights do not hurt performance on their own. 
The main cause of degradation is a high sparsity ratio.
 Percentile Centering lets us shift pre-activations back into the positive range, 
 and we observe no performance improvement from doing so (\S\ref{sec:sparsity_predictor}).
  We therefore have no evidence that a drift-free activation function would help, 
  although the question remains open.

\paragraph{Spikes and special tokens in vision transformers.}
If early-layer spikes in language models originate from tokens with a functional role,
 a natural question is whether similar spikes localize on the explicit functional tokens of 
 vision transformers. The [CLS] token, or register tokens,
  which \citet{darcet2024registers} introduce to absorb the high-norm artifacts that otherwise 
  appear in background patches. 
  Profiling MLP inputs per token in ViTs with and without registers would directly test this hypothesis.
   We leave this to future work.

\begin{figure*}[h]
    \centering
    \includegraphics[angle=90,width=\textwidth]{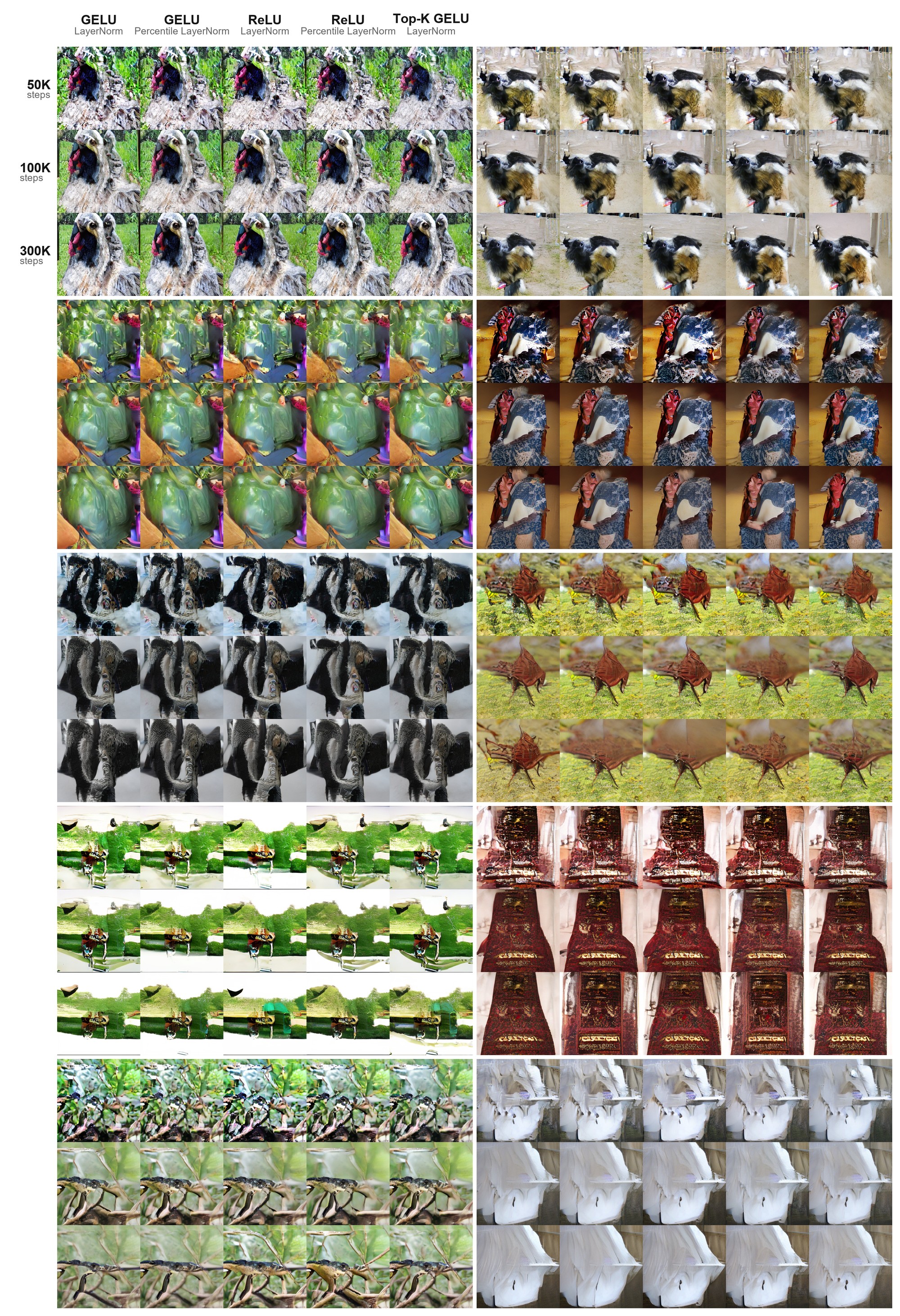}
    \caption{\textbf{DiT-S/2 generated samples across configurations
        at 50K, 100K, and 300K training steps.}
        Columns correspond to GELU + LayerNorm, GELU + Percentile
        LayerNorm, ReLU + LayerNorm, ReLU + Percentile LayerNorm,
        and Top-K GELU + LayerNorm. All configurations produce
        coherent compositions by 50K steps. High-frequency
        differences become apparent at later checkpoints: the
        baseline ReLU configuration produces blurred backgrounds
        and lacks fine detail, while Top-K GELU preserves complex
        structural features such as grasshopper limbs and
        environmental textures.}
    \label{fig:dit_generations}
\end{figure*}


\section{Technical Details}
\label{app:technical_details}

This appendix provides full implementation details for all experiments
in the paper. \S\ref{sec:mlp_details} covers MLP experiments,
\S\ref{sec:resnet_details} ResNet-18, \S\ref{sec:dit_details}
DiT-S/2, \S\ref{sec:maxvit_details} MaxViT-T, and
\S\ref{app:gpt_details} GPT-nano. Shared dataset configurations
are described in \S\ref{sec:common_setup}. Activation statistics
measurement protocols are defined in \S\ref{app:statistics}.

\subsection{Shared dataset and training configuration}
\label{sec:common_setup}

\paragraph{CIFAR-10.}
Unless otherwise specified, classification experiments use CIFAR-10
with per-channel normalization
$(\mu=[0.4914, 0.4822, 0.4465],\,\sigma=[0.2470, 0.2435, 0.2616])$,
a training batch size of 128, and evaluation on an unaugmented test
set with batch size 100. All linear and convolutional layers omit
bias terms.

\paragraph{ImageNet-1K.}
Generative and large-scale classification experiments use
ImageNet-1K. Images are center-cropped and resized to
$256\times256$ (DiT) or $224\times224$ (MaxViT) with standard
per-channel normalization.

\paragraph{Activation statistics.}
\label{app:statistics}
Two statistics are measured on a fixed test batch throughout all
experiments:
\begin{itemize}
    \item \textbf{Negative percentage}: fraction of pre-activation
    values (linear layer outputs) that are negative,
    $\text{NegPct}_l = \frac{1}{N_l}\sum_i \mathbb{I}(x_{l,i} < 0)$.
    \item \textbf{Sparsity}: fraction of post-activation outputs
    with magnitude below $\epsilon = 10^{-7}$, averaged across
    the full test set.
\end{itemize}

\paragraph{Normalization variants.}
Where normalization is varied, three strategies are evaluated in
pre-norm configuration (Norm $\to$ Linear $\to$ Act):
\begin{itemize}
    \item \textbf{No normalization}: baseline without any
    normalization.
    \item \textbf{RMSNorm}~\citep{zhang2019root}: normalizes by
    root mean square without centering,
    $\text{RMSNorm}(x) = x\,/\,\sqrt{\frac{1}{D}\sum_i x_i^2
    + \epsilon} \odot w$, where $w \in \mathbb{R}^D$ is a learned
    scale and $\epsilon = 10^{-6}$.
    \item \textbf{LayerNorm}~\citep{ba2016layer}: centers and
    scales activations,
    $\text{LayerNorm}(x) = (x - \mu_x)\,/\,\sqrt{\sigma_x^2
    + \epsilon} \odot w + b$, with learned affine parameters
    $w, b \in \mathbb{R}^D$.
\end{itemize}
Each normalization variant is evaluated with six activation
functions: ReLU, GELU, SiLU, ReLU$^2$, NoisyReLU, and
SUGARBSiLU. Higher-order and clipped variants (GELU$^2$,
ReLU$^2_{\text{clip15}}$, ReLU$^2_{\text{clip50}}$) are
additionally evaluated in GPT-nano experiments
(\S\ref{app:gpt_details}).

\paragraph{Hardware for Accumulation Stop experiments.}
DiT-S/2 experiments are conducted on Nvidia H200 (141\,GB) 
GPUs using ImageNet-1K images resized to $256{\times}256$. 
MaxViT experiments use Nvidia A100 (80\,GB) GPUs. All models 
are trained in bfloat16 mixed precision with percentile and 
centering thresholds computed channel-wise and averaged over 
the batch. The baseline uses PyTorch's \rev{optimized} LayerNorm 
and BatchNorm kernels; our AS implementation uses plain 
PyTorch without custom CUDA kernels, so all reported 
throughput gains are conservative lower bounds.
\label{app:as_hardware}

\subsection{MLP experiments}
\label{sec:mlp_details}

We use a 5-layer fully connected MLP without normalization layers:
an input projection followed by 5 repeated blocks of the form
$\text{Linear}(d,d) \to [\text{Act}]$, and a final linear layer.

\paragraph{CIFAR-10 (cross-entropy).}
Input dimension 3072, hidden dimension $d=1024$, output dimension
10. Trained for 5 epochs with Adam ($\text{lr}=10^{-3}$, weight
decay $10^{-4}$). Trajectories averaged across 5 runs
(Figure~\ref{fig:mlp_ce_cov}).

\paragraph{Random data (MSE).}
To isolate the effect of activation functions on weight drift
independently of the data distribution, models are trained on
random $\{X,Y\}$ pairs sampled from $\mathcal{N}(0,1)$ with
MSE loss. Hidden dimension $d=128$, trained for 5 epochs with
SGD ($\text{lr}=0.01$). Trajectories averaged across 5 runs
(Figure~\ref{fig:mlp_weight_drift}).

\subsection{ResNet-18 experiments}
\label{sec:resnet_details}

All ResNet-18 experiments use the standard architecture with
BasicBlock units and a 64--128--256--512 channel progression,
trained on CIFAR-10 following Section~\ref{sec:common_setup}.

\paragraph{Activation function comparison.}
Six activation functions (ReLU, GELU, SiLU, ReLU$^2$, NoisyReLU,
SUGARBSiLU) are evaluated with standard Batch Normalization.
Training uses SGD with Nesterov momentum (0.9), weight decay
$5\times10^{-4}$, and a learning rate schedule of 5-epoch linear
warmup (from $10^{-4}$ to $10^{-3}$) followed by cosine annealing.

\paragraph{Top-K sparsified GELU.}
Explicit sparsity is applied via a Top-K mask after GELU:

\begin{equation*}
\begin{aligned}
    \text{TopK-GELU}(x) &= \text{GELU}(x) \odot M, \\
    M_{i,j} &= \mathbb{I}\!\left(|\text{GELU}(x)_{i,j}|
    \in \mathrm{top}_K(|\text{GELU}(x)|)\right).
\end{aligned}
\end{equation*}

Five retention levels are evaluated: $K \in
\{0.01, 0.05, 0.10, 0.15, 0.20\}$. Training uses SGD with
Nesterov momentum (0.9), weight decay $5\times10^{-4}$, initial
learning rate $5\times10^{-2}$, and 5-epoch linear warmup
followed by cosine annealing, for 50 epochs.

\subsection{DiT-S/2 configuration}
\label{sec:dit_details}

\paragraph{Architecture.}
We adopt DiT-Small (DiT-S/2)~\citep{peebles2023dit},
consisting of 12 transformer blocks ($D=384$, 6 heads) with
adaptive LayerNorm (adaLN-Zero) for timestep and class
conditioning. All linear layers omit bias terms.

\paragraph{VAE latent space.}
The VAE encoder from Stable Diffusion v1.5 projects
$256\times256\times3$ RGB images into a $32\times32\times4$
latent space. Latents are scaled by $0.18215$ during training
and generation.

\paragraph{Diffusion training and sampling.}
The model is trained for ${\sim}300$K steps (batch size 256)
to minimize the $\ell_2$ denoising objective:
\begin{equation}
    \mathcal{L} = \mathbb{E}_{t,x_0,\epsilon}
    \left[\left\|\hat{\epsilon}_\theta(x_t,t,y)
    - \epsilon\right\|_2^2\right],
\end{equation}
where $t \sim \text{Uniform}(0,249)$ and noise is added via a
squared cosine schedule. Generation uses the EMA model with
Classifier-Free Guidance ($s=1.5$) and 250 sampling steps.

\paragraph{Accumulation Stop for normalization statistics.}
To decouple training dynamics from fluctuating batch statistics
while maintaining efficiency, we implement an Accumulation Stop
(AS) mechanism using an Exponential Moving Average (EMA):
\begin{equation}
    \bar{v}_i = \gamma\bar{v}_{i-1} + (1-\gamma)v_i,
    \quad \gamma = 0.9999.
\end{equation}
After $K=50{,}000$ steps, the running statistic $\bar{v}_K$ is
frozen for all subsequent iterations. Note that Percentile
BatchNorm with AS effectively functions as static Percentile
BatchNorm after step $K$, analogously to standard BatchNorm
with AS transitioning to static normalization.

\subsection{MaxViT-T configuration}
\label{sec:maxvit_details}

\paragraph{Architecture and training.}
MaxViT-T~\citep{tu2022maxvit} uses a 2--2--5--2 block structure
with interleaved local and global attention. Training runs for
60 epochs with AdamW ($\text{lr}=6\times10^{-4}$, weight decay
0.1) under a OneCycleLR schedule. Data augmentation follows
the original RandAugment protocol ($N=2$, $M=9$).

\paragraph{Latency benchmarking protocol.}
To isolate the per-forward-pass overhead of different activation
and normalization configurations, we employ a four-phase
measurement protocol (100 steps each):
\begin{enumerate}
    \item \textbf{Warm-up}: GPU synchronization and cache
    stabilization.
    \item \textbf{Active Train}: standard training with dynamic
    normalization statistic updates.
    \item \textbf{Converged Train}: training with Accumulation
    Stop (frozen EMA statistics).
    \item \textbf{Inference}: evaluation mode, no gradients.
\end{enumerate}
All throughput figures are measured in bfloat16 mixed precision
with statistics computed channel-wise and averaged over the batch.
The baseline uses PyTorch's optimized kernel; our AS
implementation uses plain PyTorch, so reported throughput
recovery represents a conservative lower bound.

\subsection{GPT-nano training details}
\label{app:gpt_details}

\paragraph{Architecture.}
We train a GPT-nano model (124M parameters) based on the
\texttt{nanoGPT} implementation~\citep{karpathy2022nanogpt},
with 12 layers, 12 attention heads, embedding dimension 768,
and vocabulary size 50,257.

\paragraph{Training configuration.}
The model is trained on FineWeb using the Muon optimizer for
linear layers and Adam for all other parameters. Key
hyperparameters: batch size 16, sequence length 2048, 16
gradient accumulation steps (totaling 7,168 effective
iterations), zero weight decay, optimizer betas $[0.9, 0.95]$.
The learning rate schedule consists of 2,000 warmup and 2,028
warmdown iterations.

\paragraph{Evaluation.}
\rev{Test evaluation} is performed every 1,024 iterations on
1,048,576 tokens with \rev{an evaluation} batch size of 32. \rev{No hyperparameters are selected on the test split.}
Training time per model is approximately 8 GPU hours.
A total of 23 configurations are evaluated, covering
activation functions (ReLU, GELU, NoisyReLU, SUGARBSiLU,
ReLU$^2$, GELU$^2$, ReLU$^2_{\text{clip15}}$,
ReLU$^2_{\text{clip50}}$), sparsification strategies
(Top-K GELU at multiple retention levels), and normalization
variants (LayerNorm, QuantileLayerNorm at multiple percentiles).

\subsection{Activation function implementations}
\label{app:activation_implementations}

All custom activation functions are implemented as \texttt{nn.Module}
subclasses and are drop-in replacements for standard PyTorch
activations. We describe each below alongside the key implementation
decisions.

\paragraph{ReLU$^2$.}
ReLU$^2$ is defined as $f(x) = \max(0,x)^2$ and is implemented as
a simple composition of ReLU and a squaring operation:

\begin{lstlisting}[language=Python]
class ReLUSquared(nn.Module):
    def forward(self, x):
        return F.relu(x).square()
\end{lstlisting}

\noindent No learnable parameters are introduced. The squaring is
applied \emph{after} the ReLU threshold, so the zero-output region
is preserved exactly and sparsity behavior is identical to ReLU at
the activation boundary.

\paragraph{NoisyReLU.}
NoisyReLU~\citep{gulcehre2016noisy} addresses the dying neuron
problem by injecting input-dependent noise into negative
pre-activations during training, while reverting to standard ReLU
at inference to preserve sparsity. Formally:
\begin{equation}
    f(x) = \alpha \cdot \text{ReLU}(x) + (1-\alpha)\cdot x
    + \epsilon(x),
\end{equation}
where $\alpha \in [0,1]$ interpolates between ReLU and the identity,
and the noise term is:
\begin{align*}
    \epsilon(x) &= \mathbb{I}[x \leq 0]\cdot\sigma\cdot|\mathcal{N}(0,1)|, \\
    \sigma &= c\cdot\left(\text{sigmoid}(v\cdot\delta) - 0.5\right)^2.
\end{align*}
with $\delta = \max(0,-x)$ the magnitude of the negative input and
$v$ a learnable scalar parameter. The noise magnitude $\sigma$ is
therefore input-dependent and learned, vanishing as $x \to 0^-$ and
growing for more negative inputs. At test time the function reduces
exactly to ReLU:

\lstset{
  basicstyle=\ttfamily\footnotesize,
  breaklines=true,
  breakatwhitespace=false,
  columns=fullflexible,
  keepspaces=true,
}

\begin{lstlisting}[language=Python]
class NoisyReLU(nn.Module):
    def __init__(self, alpha=1.0, c=1.0):
        super().__init__()
        self.alpha = alpha
        self.c = c
        self.p = nn.Parameter(torch.randn(1))
    def forward(self, x):
        if not self.training:
            return F.relu(x)
        delta = torch.where(
            x < 0, -x, torch.zeros_like(x))
        sigma = self.c * (
            torch.sigmoid(self.p * delta) - 0.5
        ).square()
        noise = torch.where(
            x < 0,
            sigma * torch.abs(torch.randn_like(x)),
            torch.zeros_like(x))
        return (self.alpha * F.relu(x)
                + (1 - self.alpha) * x + noise)
\end{lstlisting}

\noindent In all experiments we use $\alpha=1.0$ and $c=1.0$ with
half-normal noise, i.e.\ $\epsilon = |\mathcal{N}(0,1)|$, so the
identity branch vanishes and only the noise term acts on negative
pre-activations.

\paragraph{SUGARBSiLU.}
SUGARBSiLU~\citep{horuz2025resurrection} is a surrogate gradient method
that decouples the forward and backward passes. The forward pass
applies standard ReLU, preserving hard sparsity at every step,
while the backward pass substitutes the zero gradient of dead
neurons with the derivative of B-SiLU --- a smooth surrogate that
maintains non-zero gradient signal for negative pre-activations:
\begin{align}
    f_{\text{fwd}}(x) &= \text{ReLU}(x), \\
    \left.\frac{\partial f}{\partial x}\right|_{\text{bwd}}
    &= \sigma(x) + (x+\alpha)\cdot\sigma(x)(1-\sigma(x)),
\end{align}
where $\sigma(x) = (1+e^{-x})^{-1}$ is the sigmoid function and
$\alpha = 1.67$ is a fixed shift parameter. This is implemented via
a custom \texttt{torch.autograd.Function} that applies ReLU in
\texttt{forward()} and the B-SiLU derivative in \texttt{backward()},
requiring no changes to the model architecture or inference pipeline:

\begin{lstlisting}[language=Python]
ALPHA = 1.67

class SUGARBSiLUFunction(torch.autograd.Function):
    @staticmethod
    def forward(ctx, x):
        ctx.save_for_backward(x)
        return F.relu(x)

    @staticmethod
    def backward(ctx, grad_output):
        x, = ctx.saved_tensors
        sigma = torch.sigmoid(x)
        surrogate = sigma + (x + ALPHA) * sigma * (1.0 - sigma)
        return grad_output * surrogate

class SUGARBSiLU(nn.Module):
    def forward(self, x):
        return SUGARBSiLUFunction.apply(x)
\end{lstlisting}

\noindent The surrogate gradient is strictly positive for all $x$,
ensuring that no neuron becomes permanently dead during training
regardless of how negative its pre-activation becomes. At inference,
the \texttt{forward()} path is used exclusively, so the model
produces the same sparse activations as a standard ReLU network.

\end{document}